\documentclass[11pt,twoside,final]{article}

\usepackage{hyperref}
\hypersetup{colorlinks,linkcolor={blue},citecolor={blue},urlcolor={red}} 

\usepackage{fullpage}

\usepackage[mathscr]{eucal}
\usepackage{amsbsy}

\DeclareMathAlphabet{\mathpzc}{OT1}{pzc}{m}{it}

\usepackage{epsf}
\usepackage{fancyheadings}
\usepackage{graphics}
\usepackage{graphicx}
\usepackage{enumitem}
\usepackage{float}

\usepackage{color}

\usepackage{amsthm}
\usepackage{amsfonts}
\usepackage{amsmath}
\usepackage{amssymb}

\usepackage{algorithm}
\usepackage{mathabx}
\usepackage{algorithmic}
\usepackage{bbm}

\usepackage{caption}
\usepackage{subcaption}

\usepackage[square, numbers]{natbib}

\usepackage{macros}

\newcommand{\Vtil}{\ensuremath{\tilde{V}}}

\newcommand{\Qvaluesbar}{\widebar{Q}}

\newcommand{\deltaTol}{\ensuremath{\pardelta_{0}}}

\newcommand{\valCov}{\SigMat^\star_{\texttt{val}}(\valuestar)}

\newcommand{\Bmat}{\ensuremath{\mathbf{B}}}

\newcommand{\specfast}{\ensuremath{\varphi_f}}
\newcommand{\specslow}{\ensuremath{\varphi_s}}

\newcommand{\ErrEst}{\ensuremath{\mathscr{E}_\numobs}}

\newcommand{\Vhat}{\valuehat}

\newenvironment{carlist}
 {\begin{list}{$\bullet$}
 {\setlength{\topsep}{0in} \setlength{\partopsep}{0in}
  \setlength{\parsep}{0in} \setlength{\itemsep}{\parskip}
  \setlength{\leftmargin}{0.07in} \setlength{\rightmargin}{0.08in}
  \setlength{\listparindent}{0in} \setlength{\labelwidth}{0.08in}
  \setlength{\labelsep}{0.1in} \setlength{\itemindent}{0in}}}
 {\end{list}}

\newcommand{\bcar}{\begin{carlist}}
\newcommand{\ecar}{\end{carlist}}

\newcommand{\PEbatch}{\ensuremath{N}}

\newcommand{\vfast}{\ensuremath{\hat{\varepsilon}_f}}
\newcommand{\vslow}{\ensuremath{\hat{\varepsilon}_s}}

\newcommand{\doublehold}[1]{\ensuremath{h_{#1}}}

\newcommand{\ErrEstQ}{\ensuremath{\mathscr{E}_\numobs}}
\newcommand{\vhat}{\ensuremath{\widehat{V}}}
\newcommand{\order}{\ensuremath{\mathcal{O}}}

\newcommand{\plaincon}{\ensuremath{c}}

\begin{document}


\begin{center}

  {\bf{\LARGE{Instance-Dependent Confidence and Early Stopping for
        Reinforcement Learning}}}
\vspace*{.2in}

{\large{
\begin{tabular}{ccc}
Koulik Khamaru$^{\dagger, \ddagger}$  Eric Xia$^{\dagger, \ddagger}$ \\ Martin
 J. Wainwright$^{\dagger, \star}$ \quad Michael I. Jordan$^{\dagger,
   \star}$
\end{tabular}
}}

\vspace*{.2in}

\begin{tabular}{c}
Department of Statistics$^\dagger$, and \\ Department of Electrical
Engineering and Computer Sciences$^\star$ \\ UC Berkeley, Berkeley, CA
94720 \\
\end{tabular}

\vspace*{.2in}

\let\thefootnote\relax\footnotetext{$\ddagger$ Eric Xia and Koulik Khamaru contributed equally to this work}

\today

\vspace*{.2in}

\begin{abstract}
  Various algorithms for reinforcement learning (RL) exhibit dramatic
  variation in their convergence rates as a function of problem
  structure. Such problem-dependent behavior is not captured by
  worst-case analyses and has accordingly inspired a growing effort in
  obtaining instance-dependent guarantees and deriving
  instance-optimal algorithms for RL problems. This research has been
  carried out, however, primarily within the confines of theory,
  providing guarantees that explain \textit{ex post} the performance
  differences observed. A natural next step is to convert these
  theoretical guarantees into guidelines that are useful in
  practice. We address the problem of obtaining sharp
  instance-dependent confidence regions for the policy evaluation
  problem and the optimal value estimation problem of an MDP, given
  access to an instance-optimal algorithm.  As a consequence, we
  propose a data-dependent stopping rule for instance-optimal
  algorithms.  The proposed stopping rule adapts to the
  instance-specific difficulty of the problem and allows for early
  termination for problems with favorable structure.
\end{abstract}

\end{center}


\section{Introduction}

Reinforcement learning (RL) refers to a broad class of methods that
are focused on learning how to make (near)-optimal decisions in
dynamic environments.  Although RL-based methods are now being
deployed in various application domains (e.g.,~\citep{tobin2017domain,
  JMLR:v17:15-522, silver2016alphago}), such deployments often lack a
secure theoretical foundation.  Given that RL involves making
real-world decisions, often autonomously, the impact on humans can be
significant, and there is an urgent need to shore up the foundations,
providing practical and actionable guidelines for RL.  A major part of
the challenge is that popular RL algorithms exhibit a variety of
behavior across domains and problem instances and existing methods and
associated guarantees, generally tailored to the worst-case setting,
fail to capture this variety.  One way to move beyond worst-case
bounds is to develop guarantees that adapt to the problem difficulty,
helping to reveal what aspects of an RL problem make it an ``easy'' or
``hard'' problem. Indeed, in recent years, such a research agenda has
begun to emerge and we have gained a refined understanding of the
instance-dependent nature of various reinforcement learning
problems~\citep[e.g.,][]{simchowitz2019non,zanette2019tighter,zanette2019almost,maillard2014hard,khamaru2020PE,pananjady2021instance}.

The broader challenge---and the focus of this paper---is to recognize
that RL involves decision-making under uncertainty, and to develop an
inferential theory for RL problems.  Such theory must not only be
\emph{instance-dependent}, but also \emph{data-dependent}, meaning
that quantities such as confidence intervals should be computable
using available data.  The latter property is not shared by most past
instance-dependent guarantees in RL; with limited exceptions (e.g.,
Theorem 1(a) in the paper~\cite{pananjady2021instance}), most results
from past work depend on population-level objects--such as probability
transition matrices, Bellman variances, or reward function
bounds---that are not known to the user.  Due to this lack of
data-dependence, such results, while theoretically useful, cannot
actually be used by the practitioner in guiding the process of data
collection and inference.

Our work aims to close this gap between theory and practice by
providing theoretical guarantees that are both instance and
data-dependent.  The resulting bounds are both sharp and computable
based on data, and we use our theory to design early stopping
procedures that output estimates with confidence intervals of
prescribed widths and coverage.  These procedures lead to substantial
reductions in the amount of data required for a target accuracy.  In
more detail, we make these contributions in the context of Markov
decision processes (MDPs) with a finite number of states and actions,
and problem-dependent confidence regions both for policy evaluation
and optimal value function estimation.  Contrary to prior work on
instance-dependent analysis, our work allows a user to adapt their
data requirements for the specific MDP at hand by exploiting the local
difficulty of the MDP.  As we show, doing so can lead to significant
reductions in the sample sizes required for effective learning.


\subsection{Related work}

The problem of estimating the value function for a given policy in a
Markov decision process (MDP) is a key subroutine in many modern-day
RL algorithms. Examples include policy iteration~\citep{howarddp},
policy gradient, and actor-critic methods~\citep{reinforce1992,
  konda2001actorcritic, silver2014policygrad, pmlr-v48-mniha16}. Such
use cases have provided the impetus for the recent interest in
analyzing policy evaluation. Much of the focus in the past has been on
understanding TD-type algorithms with instance-dependent analyses:
function approximation under the $\ell_2$
error~\citep{pmlr-v75-bhandari18a,dalal2018finite, xu2020reanalysis},
tabular setting under the
$\ell_\infty$-error~\citep{khamaru2021instance,pananjady2021instance},
or under kernel function approximation~\citep{duan2021optimal}. Many
of these results established instance-specific guarantees that improve
upon global worst-case bounds~\citep{Azar2013Minimax}. In particular,
the paper~\cite{khamaru2020PE} establishes a local minimax lower-bound
in the tabular setting and proposes a procedure that achieves it.

Policy optimization involves solving for an optimal policy within a
given MDP. There exists a variety of different techniques for solving
policy optimization; the one of interest here is $Q$-learning,
introduced in the paper~\citep{Watkins1992Qlearning}. There has been
much prior work on the theory of $Q$-learning, such as convergence of
the standard updates~\citep{wainwright2019stochastic,
  pmlr-v139-li21b}, global minimax lower bounds for estimation of
optimal $Q$-functions~\cite{Azar2013Minimax}, variance-reduced
versions of $Q$-learning and their worst-case
optimality~\cite{sidford2018variance, sidford2018near,
  wainwright2019variancereduced}, and the asynchronous
setting~\citep{Li2020SampleCO}. The
paper~\citep{khamaru2021instance} establishes the local
non-asymptotic minimax lower bound of estimating the $Q$-function and
proves that variance-reduced $Q$-learning achieves it.


\subsection{Contributions}

In brief, the main contributions of this paper are to provide
guarantees for policy evaluation and optimization that are both
instance-optimal and data-dependent.  Our first main result, stated as
Theorem~\ref{thm:main-thm-poleval}, applies to a meta-procedure, which
takes as input any base procedure for policy evaluation that is
instance-optimal up to constant factors.  It leverages this base
procedure using a portion of the dataset, and uses the resulting
output along with the remaining data to compute bounds on the
estimated value function.  We prove that these bounds---which are
data-dependent bounds by construction---are also instance-optimal up
to constant factors.  Thus, they can inverted so as to produce a
confidence region for the value function, and up to constant factors,
the width of this confidence region is as small as possible for a
pre-specified coverage.  Next, based on the guarantees from
Theorem~\ref{thm:main-thm-poleval}, we introduce an early stopping
protocol for policy evaluation, known as the~\ref{EmpIRE} procedure,
that is guaranteed to output instance-optimal confidence regions up to
constant factors upon stopping. As we show both theoretically and in
simulation, use of this early stopping procedure can lead to
significant reductions in sample sizes relative the use of worst-case
criteria.

\begin{figure}[h]
  \begin{center}
    \begin{tabular}{cc}
  \widgraph{0.45\textwidth}{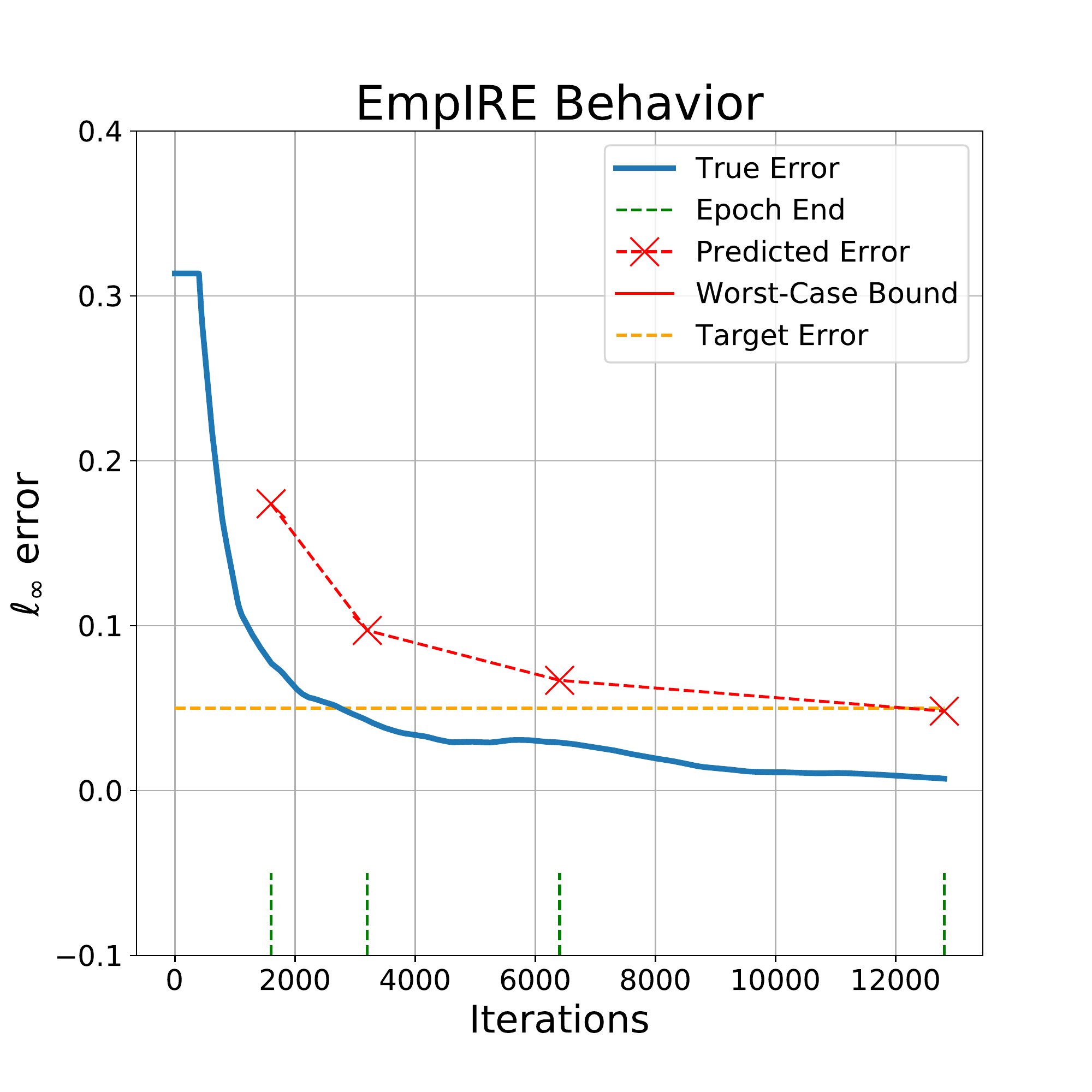} &
  \widgraph{0.45\textwidth}{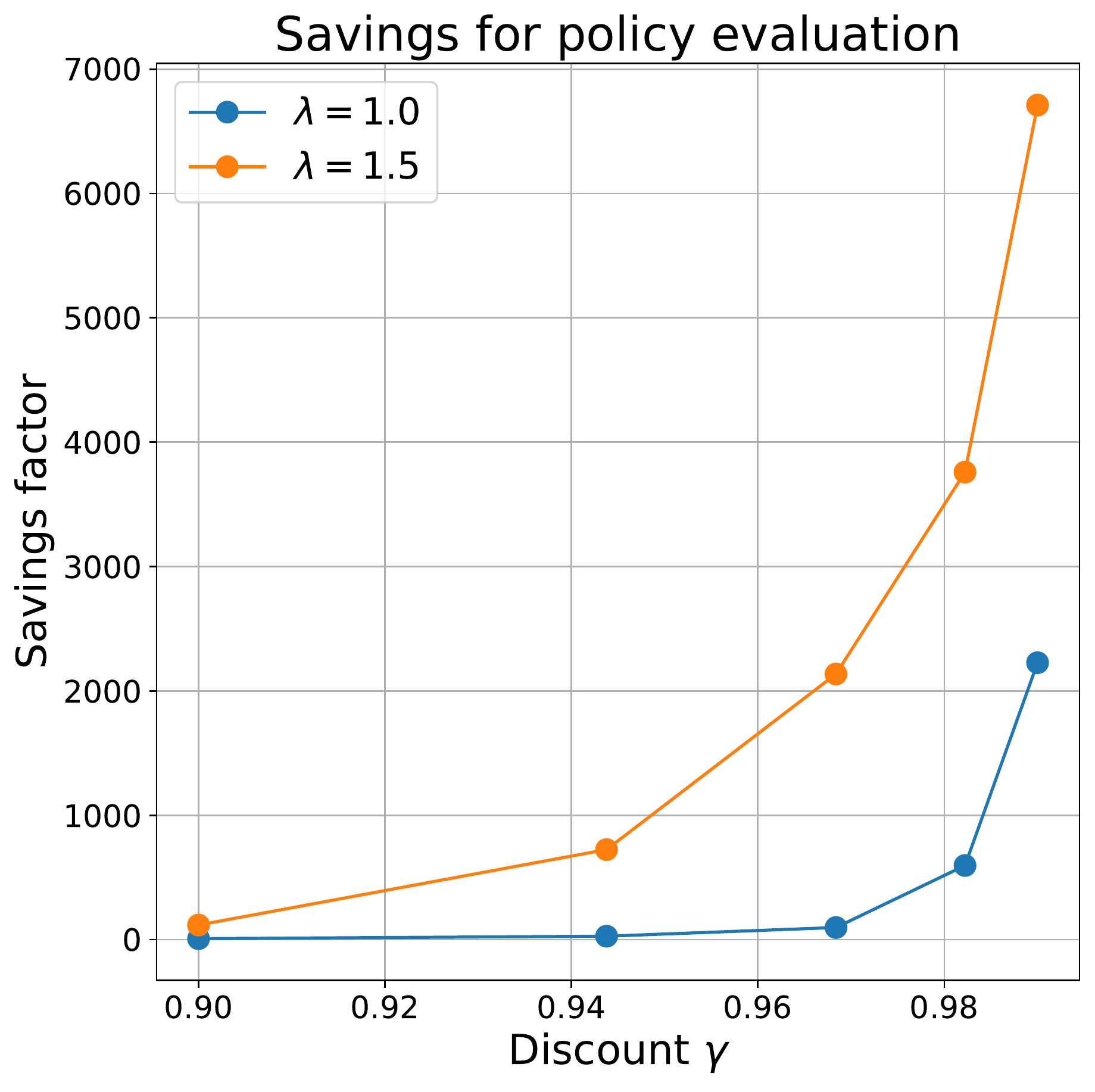} \\
  (a) & (b)
    \end{tabular}
    \caption{(a) Illustration of the behavior of the~\ref{EmpIRE}
      early stopping protocol when combined with an instance-optimal
      procedure for estimating value functions. The true error (blue
      line) of the value function estimate is plotted as a function of
      the number of samples used (or equivalently, iterations of the
      algorithm).  The~\ref{EmpIRE} protocol checks the error at the
      end of a dyadically increasing sequence of epochs, as marked
      with vertical green lines; the associated error estimates are
      marked with red X's.  For this particular run, the~\ref{EmpIRE}
      protocol terminates at the end of the $4^{th}$ epoch when the
      predicted error is smaller than the target error (orange
      line). (b) Illustration of the savings in sample size afforded
      by the~\ref{EmpIRE} protocol for different choices of
      $\contractPar$ and $\diff$. Plots show the ratio of worst-case
      sample size over the actual sample size (vertical axis) versus
      the log discount complexity (horizontal axis).}
    \label{FigFront}
  \end{center}
\end{figure}

Figure~\ref{FigFront} gives a preview of the results to come,
including the behavior of this early stopping procedure (panel (a)),
along with the attendant benefits of substantially reduced sample
sizes (panel (b)).  The \ref{EmpIRE} method is an epoch-based
protocol: within each epoch, it uses all currently available samples
to estimate the value function along with its associated
$\ell_\infty$-error, and it terminates when the error estimate drops
below a pre-specified target.  Our guarantees ensure that with high
probability over all epochs, the error estimate is an upper bound on
the true error, so that the final output of the procedure has
guaranteed accuracy with the same probability.
Figure~\ref{FigFront}(a) illustrates the behavior of the~\ref{EmpIRE}
procedure over a simulated run.  The predicted error incorporates the
instance-dependent structure of the MDP at hand and allows us to
terminate the procedure earlier for ``easy''
problems. Figure~\ref{FigFront}(b) highlights how difficulty can vary
dramatically across different instances.  We do so by constructing a
class of MDPs for which the difficulty can be controlled by a
parameter $\lambda$, with larger $\lambda$ indicating an easier
problem. For a given accuracy $\epsilon$, let $\numobs(\epsilon)$ be
the number of samples required to achieve an estimate with this
accuracy. We can bound this sample complexity using either global
minimax theory (based on worst-case assumptions), and compare it to
the instance-dependent results of the~\ref{EmpIRE} procedure.
Figure~\ref{FigFront}(b) plots the ratio of the worst-case prediction
(from global minimax) to the number of samples used by~\ref{EmpIRE} as
a function of the discount parameter $\discount$, for two different
choices of the hardness parameter $\lambda$.  We see that~\ref{EmpIRE}
can yield dramatic reductions in sample complexity compared to a
worst-case guarantee---on the order of $10^3$ for larger discounts.

We also derive similar guarantees in the more challenging setting of
policy optimization, where we again analyze a meta-procedure that
takes as input any algorithm that returns an instance-optimal estimate
of the optimal $Q$-value function.
Theorem~\ref{thm:general-CI-Policy-opt} gives the resulting
data-dependent and instance-dependent guarantees enjoyed by this
procedure, and inverting these bounds again leads confidence regions
for the optimal value function.  As before, these guarantees can be
combined with the~\ref{EmpIRE} protocol so as to perform early
stopping while still retaining theoretical guarantees for policy
optimization.

The remainder of this paper is organized as follows.
Section~\ref{SecPolEval} is devoted to results on the policy
evaluation problem, and Section~\ref{SecPolOpt} discusses results
related to policy optimization via optimal value function estimation.
In Section~\ref{SecProofs}, we provide the proofs of our main
results, with some more technical results deferred to the appendices.
We conclude with a discussion in Section~\ref{SecDiscussion}.


\subsection{Notation}  

For a positive integer $n$, let $[n] \defn \{1, 2, \ldots, n\}$. For a
finite set $S$, we use $|S|$ to denote its cardinality. We let
$\mathbf{1}$ denote the all-ones vector in $\RR^{\dims}$. Let $e_j$
denote the $j^{th}$ standard basis vector in $\RR^{\dims}$. For a
vector $u \in \RR^{\dims}$, we use $| u| $ to denote the entrywise
absolute value of a vector $u \in \RR^{\dims}$; square-roots of
vectors are, analogously, taken entrywise. Given two vectors $u, v$ of
matching dimensions, we use $u \succeq v$ to indicate that $u - v$ is
entrywise non-negative and define $u \preceq v$ analogously. For a
given matrix $A \in \RR^{\dims \times \dims}$ we define the diagonal
``norm'' $\diagnorm{A} = \max_{i=1, \ldots, \dims} \big| e_i^T A e_i
\big|$, i.e., the maximum diagonal entry in absolute terms. We use $\lesssim$
and $\bigOh(\cdot)$ to denote relations that hold up to constant and logarithmic factors.


\section{Some background}
\label{SecBackground}

In this section, we provide some background on Markov decision
processes (MDPs), policy evaluation, and optimal value estimation
problems.


\subsection{Markov decision processes}

We start with a brief introduction to Markov decision processes (MDPs)
with finite state $\stateset$ and action $\actionset$ spaces; see the
books~\cite{puterman2014markov, Bertsekas2009, Sutton1998} for a more
in-depth background. In a Markov decision process, the state $\state$
evolves dynamically in time under the influence of the actions.
Concretely, there is a collection of probability transition kernels,
$\{\TranMatQ_\action(\cdot\mid\state) \mid (\state, \action) \in
\stateset \times \actionset\}$, where $\TranMatQ_\action(\state' \mid
\state)$ denotes the probability of a transition to the state
$\state'$ when the action $\action$ is taken at the current state
$\state$.  In addition, an MDP is equipped with a reward function
$\rewardQ$ that maps every state-action pair $(\state, \action)$ to a
real number $\rewardQ(\state, \action)$. The reward $\rewardQ(\state,
\action)$ is the reward received upon performing the action $\action$
in the state $\state$.  Overall, a given MDP is characterized by the
pair $(\TranMatQ, \rewardQ)$, along with a discount factor
$\contractPar \in (0,1)$.

A deterministic policy $\policy$ is a mapping $\stateset \to
\actionset$: the quantity $\policy(\state) \in \actionset$ indicates
the action to be taken in the state $\state$.  The value of a policy
is defined by the expected sum of discounted rewards in an infinite
sample path. More precisely, for a given policy $\policy$ and discount
factor $\contractPar \in (0, 1)$, the value function for policy
$\policy$ is given by
\begin{align}
\values^{\policy}(\state) \defn \EE \Big[ \sum_{k = 0}^\infty
  \contractPar^k \cdot \rewardQ(\state_k, \action_k) \mid \state_0 =
  \state \Big], \qquad \mbox{where $u_k = \policy(x_k)$ for all $k
  \geq 0$}.
\end{align}
A closely related object is the action-value or $Q$-function
associated with the policy, which is given by
\begin{align}
\Qvalues^{\policy}(\state, \action) \defn \EE \Big[ \sum_{k =
    0}^\infty \contractPar^k \cdot \rewardQ(\state_k, \action_k) \mid
  (\state_0, \action_0) = (\state, \action) \Big], \qquad \text{where
  $u_k = \policy(x_k)$ for all $k \geq 1$}.
\end{align}

Two core problems in reinforcement learning---and those that we
analyze in this paper---are policy evaluation and policy optimization.
In the problem of \emph{policy evaluation}, we are given a fixed
policy $\policy$, and our goal is to estimate its value function on
the basis of samples.  In the \emph{policy optimization}, our goal is
to estimate the optimal policy, along with the associated optimal
$Q$-value function
\begin{align}
\label{eqn:Optimal-Q-defn}
\Qvaluestar(\state, \action) & \defn \max_{\pi \in \PolicySet} \;
\Qvalues^{\policy}(\state, \action) \quad \mbox{for all $(\state,
  \action) \in \stateset \times \actionset$,}
\end{align} 
again on the basis of samples.  As we now describe, both of these
problems have equivalent formulations as computing the fixed points of
certain types of Bellman operators.

\subsection{Policy evaluation and Markov reward processes}
\label{sec:polEval_to_MRP}

We begin by formalizing the problem of policy evaluation.  For a given
MDP, if we fix some deterministic policy $\policy$, then the MDP
reduces to a Markov reward process (MRP) over the state space
$\stateset$. More precisely, the state evolution over time is
determined by the set of transition functions
$\{\TranMatQ_{\policy(\state)}(\cdot \mid \state) , \state \in
\stateset \}$, whereas the reward received when at state $\state$ is
given by $\rewardQ(\state, \policy(\state))$.

When the number of states is finite with $|\stateset|$, the transition
functions $\{\TranMatQ_{\policy(\state)}(\cdot \mid \state) , \state
\in \stateset \}$ and the rewards \mbox{$\{\rewardQ(\state,
  \policy(\state)) \mid \state \in \stateset\}$} can be conveniently
represented as a $\numstates \times \numstates$ matrix and a
$\numstates$ dimensional vector, respectively.  For ease of notation,
we use $\TranMat$ to denote this $\numstates \times \numstates$
matrix, and $\reward$ to denote this $\numstates$ dimensional
vector. Concretely, for any state $\state \in \stateset$, we define
\begin{align}
\label{eqn:MDP-to-MRP}
  \reward(\state) \defn \rewardQ(\state, \policy(\state)) \qquad
  \text{and} \qquad \TranMat(\state', \state) \defn
  \TranMatQ_{\policy(\state)}(\state' \mid \state),
\end{align} 
where $\TranMat(\state', \state)$ denotes the row corresponding to
$\state$ and the column corresponding to $\state'$. We will often use
$\TranMat(\state' \mid \state)$ to denote $\TranMat(\state',
\state)$. With this formulation at hand, it is clear that evaluating
the value of the policy $\policy$ for the MDP~$\MDP$ is the same as
finding the value of the MRP $\MRP = (\reward, \TranMat,
\contractPar)$ with reward $\reward$ and transition $\TranMat$ defined
in equation~\eqref{eqn:MDP-to-MRP}.

\paragraph{Bellman evaluation operator:}

Given an MRP $\MRP = (\reward, \TranMat, \contractPar)$, its value
$\valuestar$ can be obtained as the unique fixed-point of the Bellman
evaluation operator~$\Op$.  It acts on the set of value functions in
the following way:
\begin{align*}
\Op(\values)(\state) = \reward(\state) + \contractPar \sum_{\state'
  \in \stateset} \TranMat(\state' \mid \state) \values(\state') \quad
\text{for all} \;\; \state \in \stateset.
 \end{align*} 
For finite-dimensional MDPs, the value and reward functions can be
viewed as $\numstates$-dimensional vectors, and the transition
function as a $\numstates \times \numstates$ dimensional matrix. As a
result, the action of the Bellman operator $\Op$ can be written in the
following compact form
\begin{align}
\label{eqn:Bellman_Eval}
\Op(\values) = \reward + \contractPar \TranMat \values,
\end{align}
and the associated fixed point relation is given by $\valuestar =
\Op(\valuestar)$.  See the standard
books~\cite{puterman2014markov,Sutton1998,Bertsekas2009} for more
details.

\paragraph{Generative observation model for MRPs:}

In the learning setting, the pair $(\TranMat, \reward)$ is unknown and
we assume that we have access to \mbox{\emph{i.i.d.  samples}}
$\{(\NoisyReward_k, \obsmat_k)\}_{k = 1}^\numobs$ from the reward
vector $\reward$ and from the transition matrix
$\TranMat$. Concretely, given a sample index $k \in \{1, 2, \ldots,
\numobs\}$, we have for all state $\state \in \stateset$
\begin{align}
\label{generative-sample}
\obsmat_k(\cdot, \state) \sim \TranMat(\cdot \mid \state) \qquad
\text{and} \quad \Exs[\NoisyReward_k(\state)] = \reward(\state) .
\end{align}
We also assume that the deviation of the reward sample
$\NoisyReward_k(\state)$ from the true reward $\reward(\state)$ is
bounded:
\begin{align}
  \label{eqn:reward-bound}
  |\NoisyReward(\state) - \reward(\state)| \leq \rewardbound \qquad
  \text{for all} \;\; \state \in \stateset.
\end{align}
In other words, for every sample $k$, we observe for every state
$\state$ a reward $\NoisyReward_k(\state)$ and the next state $\state'
\sim \TranMat(\cdot \mid \state)$. Then we have $\obsmat_k(\state',
\state) = 1$, and the remaining entries in the row corresponding to
$\state$ are $0$. Sometimes we will denote $\obsmat_k(\state',
\state)$ as $\obsmat_k(\state' \mid \state)$.


\subsection{Policy optimization via optimal $Q$-function estimation}

Recall that the goal of policy optimization is to find an optimal
policy along with the optimal $Q$-value function (cf.
equation~\eqref{eqn:Optimal-Q-defn}).  For MDPs with finite state
space $\stateset$ and action space $\actionset$, any $Q$-value
function can be represented as an element of $\RR^{\card{\stateset}
  \times \card{\actionset}}$.

Moreover, the optimal $Q$-function $\Qvaluestar$ is the \emph{unique
fixed point} of the Bellman (optimality) operator $\BellOptOp$, an
operator on $\RR^{\card{\stateset} \times \card{\actionset}}$ given
via
\begin{align}
\label{EqnPopulationBellmanOpt}
\BellOptOp(\Qvalues)(\state, \action) \defn \reward(\state, \action) +
\contractPar \sum_{\state' \in \state} \TranMat_\action(\state' \mid
\action) \max_{\action' \in \actionset} \Qvalues(\state', \action').
\end{align} 
Given $\Qvaluestar$, an optimal policy $\policystar$ is given by
$\policystar(\state) \in \arg \max_{\action \in \actionset}
\Qvaluestar(\state, \action)$.  Again, we refer the reader to the
standard references~\citep{puterman2014markov, Sutton1998,
  Bertsekas2009} for more details.

\paragraph{Generative observation model for MRPs:}

We operate in the \textit{generative} observation model: we are given
$\Qnumobs$ i.i.d. samples of the form $\{ (\obsmatQ_k,
\NoisyRewardQ_k)\}_{k=1}^{\Qnumobs}$, where $\NoisyRewardQ_k$ is a
matrix in $\RR^{\card{\stateset} \times \card{\actionset}}$ and
$\obsmatQ_k$ is a collection of $\card{\actionset}$ matrices in
$\RR^{\card{\stateset} \times \card{\stateset}}$ indexed by
$\actionset$. We denote by $\obsmatQ_k(\state, \action)$ the
row-vector corresponding to state $\state$ and action $\action$. The
row-vector is computed via sampling from the transition kernel
$\TranMatQ_\action(\cdot \mid \state)$, independently of all other
$(\state, \action)$, and making the entry corresponding to the next
state $1$, and the remaining entries $0$. Concretely, for every state
action pair $(\state, \action) \in \stateset \times \actionset$, we
write
\begin{align}
\label{EqnGenerativeMDP}  
  \state' \sim \TranMatQ_{\action}(\cdot \mid \state) \qquad
  \text{and} \qquad \obsmatQ_{u}(\cdot \mid \state) = \mathbf{1}_{x =
    \state'}.
\end{align}

 We write $\obsmatQ_1 + \obsmatQ_2$ to indicate the collection of
 $\card{\actionset}$ matrices that are the sum of the matrices from
 $\obsmatQ_1$ and $\obsmatQ_2$ that match in the action. We also
 assume that $\NoisyRewardQ_k(\state, \action)$ is a bounded random
 variable with mean $\rewardQ(\state, \action)$ and that
 $|\NoisyRewardQ_k(\state, \action) - \rewardQ(\state, \action) | \leq
 \rewardboundQ$. Additionally, $\NoisyRewardQ_k(\state, \action)$ is
 taken to be independent of all other state-action pairs and of the
 observations $\obsmatQ_k$.


\section{Confidence intervals for policy evaluation}
\label{SecPolEval}

In this section, we propose instance-dependent confidence regions for
the policy evaluation problem.  We start with a discussion on the
problem dependent functional that determines the difficulty for the
evaluation problem. As already detailed in
Section~\ref{sec:polEval_to_MRP}, any MDP policy pair $(\MDP,
\policy)$ naturally gives rise to an MRP, and the policy evaluation
problem is equivalent to finding the value of that MRP. Accordingly,
in the rest of the section, we discuss the problem of finding the
value of a general MRP $\MRP$.


\subsection{Optimal instance dependence for policy evaluation}

Given access to a generative sample $(\NoisyReward, \obsmat)$ from an
MRP $\MRP = (\reward, \TranMat, \contractPar)$, we can define the
single-sample empirical Bellman operator as
\begin{align}
  \label{EqnEmpiricalBellman}
  \NoisyOp(\values)(\state) & \defn \NoisyReward(\state) +
  \contractPar \sum_{\state' \in \stateset} \obsmat(\state' \mid
  \state) \values(\state') \qquad \mbox{for all $\state \in
    \stateset$.}
\end{align}
By construction, for any fixed value function $\values$, the quantity
$\NoisyOp(\values)$ is an unbiased estimate of $\Op(\values)$. In the
finite-dimensional setting, $\NoisyOp(\values)$ can be considered a
random vector with entries given by $\NoisyOp(\values)(\state)$, and
we can talk about its covariance matrix. In a recent
paper~\cite{khamaru2020PE}, the authors show that the quantity that
determines the difficulty of estimating the value $\valuestar$, given
access to i.i.d. generative samples following the sampling
mechanism~\eqref{generative-sample}, is the following $\numstates
\times \numstates$ covariance matrix
\begin{align}
  \label{eqn:Opt-cov-matrix}
  \optCovMat{\valuestar} \defn (\Id - \contractPar \TranMat)^{-1}
  \valueCovOpt (\Id - \contractPar \TranMat)^{-\top}.
\end{align}
Intuitively, the covariance matrix $\valueCovOpt =
\covar(\NoisyOp(\valuestar))$~captures the noise of the empirical
Bellman operator, evaluated at the true value $\valuestar$. This term
is then compounded by powers of the discounted transition matrix
$\contractPar \TranMat$, which captures how the perturbation in the
Bellman operator propagates over time, and thus gives rise to the
matrix $(\Id-\contractPar \TranMat)^{-1}$ via the von Neumann
expansion.

In the paper~\cite{khamaru2021instance}, it is shown the functional
$\diagnorm{\optCovMat{\valuestar}}^{\frac{1}{2}}$ arises in lower
bounds---in both asymptotic and non-asymptotic settings---on the error
of any procedure for estimating the value function.  In particular,
any estimate $\Vtil_\numobs$ of the value function must necessarily
satisfy a lower bound of the form $\|\Vtil_\numobs - \Vstar\|_\infty
\succsim \frac{1}{\sqrt{\PEnumobs}}
\diagnorm{\optCovMat{\valuestar}}^{\frac{1}{2}}$.  Moreover, the
authors~\cite{khamaru2020PE} provide a practical scheme that achieves
this lower bound, modulo logarithmic factors.  In particular, in
Theorem~1 and Proposition~1 of the paper~\cite{khamaru2020PE} show
that a variance-reduced version of policy evaluation, using a total of
$\PEnumobs$ samples, returns an estimate $\valuehat_\numobs$ such that
\begin{align}
\label{eqn:VRPE-pop-bound}
\inftynorm{\valuehat_\numobs - \valuestar} & \precsim \frac{\sqrt{\log
    (|\stateset|/\pardelta)}}{\sqrt{\PEnumobs}}
\diagnorm{\optCovMat{\valuestar}}^{1/2} + \order(\frac{1}{\numobs})
\end{align}
Notably, this upper bound can be much smaller than a worst case upper
bound (see Section 3.2 in the paper~\cite{khamaru2020PE}).  In
particular, the variance-reduced PE scheme provides far more accurate
estimates for problems that are ``easier'', as measured by the size of
the functional $\diagnorm{\optCovMat{\valuestar}}^{1/2}$.


\subsection{A procedure for computing data-dependent bounds}
\label{sec:data-dependent-bound-PE}

Guarantees of the type~\eqref{eqn:VRPE-pop-bound}---showing that a
certain procedure for policy evaluation is instance-optimal---are
theoretically interesting, but may fail to be practically useful.  In
particular, the bound~\eqref{eqn:VRPE-pop-bound} depends on the
unknown population-level quantity
$\diagnorm{\optCovMat{\valuestar}}^{1/2}$, so that for any given
instance, we cannot actually evaluate the error bound achieved for a
given sample size.  Conversely, when data is being collected in a
sequential fashion, as is often the case in practice, we cannot use
this theory to decide when to stop.

In this section, we rectify this gap between theory and practice by
introducing a data-dependent upper bound on the error in policy
evaluation (PE).  In addition to the variance-reduced PE scheme
described in the paper~\cite{khamaru2021instance}, there are other
procedures that could also be used to achieve an instance-optimal
bound, such as the ROOT-OP procedure applied to the Bellman
operator~\cite{LiMouWaiJor20, MKWBJVariance22}.  Rather than focus on
the details of a specific PE method, our theory applies to a generic
class of PE procedures, as we now define.

\subsubsection{$(\specfast, \specslow)$-instance-optimal algorithms}

Suppose we have access to some procedure $\AlgoEval$ that estimates
the value function $\valuestar$ of an MRP~$\MRP$ .  More precisely,
let $\valuehat_\numobs$ denote the estimate obtained from the
algorithm~$\AlgoEval$ using $\numobs$ generative samples
$\{\NoisyReward_k, \obsmat_k\}_{k = 1}^\numobs$.  Our analysis applies
to procedures that are instance-optimal in a precise sense that we
specify here.

Given a failure probability $\pardelta \in (0,1)$, let $\specfast$ and
$\specslow$ be non-negative functions of $\pardelta$; depending on the
procedure under consideration, these functions may also involve other
known parameters (e.g. number of states, the discount factor
$\contractPar$, initialization of the algorithm $\AlgoEval$ etc.), but
we omit such dependence to keep our notation streamlined.  For any
such pair, we say that the procedure $\AlgoEval$ is \emph{$(\specfast,
\specslow)$-instance optimal} if the estimate $\valuehat_\numobs$
satisfies the $\ell_\infty$-bound
\begin{align}
\label{eqn:PolEval-Algo-Cond}
\inftynorm{\valuehat_\numobs - \valuestar} \leq
\frac{\specfast(\pardelta)}{\sqrt{\PEnumobs}} \cdot
\diagnorm{\optCovMat{\valuestar}}^{1/2} +
\frac{\specslow(\pardelta)}{\PEnumobs}.
\end{align}
with probability at least $1 - \pardelta$.  Thus, an $(\specfast,
\specslow)$-instance-optimal procedure returns an estimate that,
modulo any inflation of the error due to $\specfast(\pardelta)$ being
larger than one, achieves the optimal instance-dependent bound.  Note
that the function $\specslow$ is associated with the higher order term
(decaying as $\numobs^{-1}$), so becomes negligible for larger sample
sizes.  As a concrete example, as shown by the
bound~\eqref{eqn:VRPE-pop-bound}, there is a variance-reduced version
of policy evaluation that is instance-optimal in this sense with
$\specfast(\pardelta) \asymp \sqrt{\log(|\stateset|/\pardelta)}$. Our
results are based on the natural assumption that the functions
$\specfast$ and $\specslow$ take values uniformly lower bounded by
$1$, and are decreasing in $\pardelta$. \\

In practice, even if an estimator enjoys the attractive
guarantee~\eqref{eqn:PolEval-Algo-Cond}, it cannot be exploited in
practice since the right-hand side depends on the \emph{unknown}
problem parameters $(\reward, \TranMat, \valuestar)$.  The main result
of the following section is to introduce a data-dependent quantity
that (a) also provides a high probability upper bound on
$\|\vhat_\numobs - \valuestar\|_\infty$; and (b) matches the
guarantee~\eqref{eqn:PolEval-Algo-Cond} apart from changes to the
leading pre-factor and the higher order terms.


\subsubsection{Constructing a data-dependent bound}

In order to obtain a data-dependent bound on
$\inftynorm{\valuehat_\numobs - \valuestar}$, it suffices to estimate
the quantity $\diagnorm{\optCovMat{\valuestar}}^{1/2}$.  Our main
result guarantees that there is a data-dependent estimate
$\VarEst(\valuehat_\numobs, \MyData{})$ that leads to an upper bound
that holds with high probability, and is within constant factors of
the best possible bound.  In terms of the shorthand $\bfun(\values)
\defn \rewardbound + \contractPar \inftynorm{\values}$, our
\emph{empirical error estimate} takes the form
\begin{align}
\ErrEst(\valuehat_\numobs, \Vdata, \pardelta) & \defn \frac{2\sqrt{6}
  \cdot \specfast(\pardelta) }{\sqrt{\PEnumobs}} \cdot
\diagnorm{\VarEst(\valuehat_\numobs, \MyData{})}^{1/2} +
\frac{2\specslow(\pardelta)}{\PEnumobs} + \frac{6
  \bfun(\valuehat_\numobs)}{1 - \contractPar} \cdot \frac{\sqrt{\log(8
    \numstates / \pardelta)}}{\PEnumobs - 1}.
\end{align}
This error estimate is a function of a dataset $\MyData{} =
(\MyData{\numobs}, \MyData{\nhold})$ where: \\

\bcar
\item the base estimator acts on $\MyData{\numobs}$ with cardinality
  $\numobs$ to compute the value function estimate
  \mbox{$\valuehat_\numobs \defn \AlgoEval(\MyData{\numobs})$,}
\item the smaller set $\MyData{2 \nhold}$ corresponds to samples that
  are held-out and play a role only in the covariance estimate
  $\VarEst(\valuehat_\numobs, \MyData{})$, as detailed in
  Section~\ref{SecEmpCov} following our theorem statement. This data
  set consists of $2 \holdnumobs$ samples with $\nhold \defn 2 \lceil
  \frac{32 \cdot \log(8\numstates^2 / \pardelta)}{(1 -
    \contractPar)^2} \rceil$.
  \ecar

\vspace*{0.15in}  

\noindent With this set-up, we have the following guarantee:

\begin{theorem}
\label{thm:main-thm-poleval}
Let $\AlgoEval$ be any $(\specfast, \specslow$)-instance-optimal
algorithm (cf. relation~\eqref{eqn:PolEval-Algo-Cond}).  Then for any
$\pardelta \in (0,1)$, given a dataset $\MyData{} = (\MyData{\numobs},
\MyData{2 \nhold})$ with $\numobs \geq \specfast^2(\pardelta) \frac{24
  \cdot 8 \cdot \log(8\numstates^2 / \pardelta)}{(1 -
  \contractPar)^2}$ samples, the empirical error estimate $\ErrEst$
has the following guarantees with probability at least $1 -
3\pardelta$:
\begin{enumerate}
\item[(a)] The $\ell_\infty$-error is upper bounded by the empirical
  error estimate:
\begin{subequations}
\begin{align}
\label{eqn:general-CI-eqn1}
\inftynorm{\valuehat_\numobs - \valuestar} & \leq
\ErrEst(\valuehat_\numobs, \MyData{}, \pardelta).
\end{align}
\item[(b)] Moreover, this guarantee is order-optimal in the sense that
\begin{align}
\label{eqn:general-CI-eqn2}
\ErrEst(\valuehat_\numobs, \MyData{}, \pardelta) \leq \frac{13 \cdot
  \specfast(\pardelta)}{\sqrt{\PEnumobs}} \cdot
\diagnorm{\optCovMat{\valuestar}}^{1/2} +
\frac{2\specslow(\pardelta)}{\PEnumobs} + \frac{32\bfun(\valuestar)}{1
  - \contractPar} \cdot \frac{\sqrt{\log(\tfrac{8
      \numstates}{\pardelta})}}{\PEnumobs - 1}.
\end{align}
\end{subequations}
\end{enumerate}
\end{theorem}

\noindent We prove this theorem in
Section~\ref{sec:proof-of-thm-poleval}. \\

A few comments regarding Theorem~\ref{thm:main-thm-poleval} are in
order.  Since the empirical error estimate $\ErrEst$ can be computed
based on the data, we obtain a data-dependent confidence interval on
$\inftynorm{\valuehat_\numobs - \valuestar}$.  More precisely, we are
guaranteed the inclusion
\begin{align}
\label{EqnConfInt}
\Big[\Vhat_\numobs(\state) - \ErrEst, \Vhat_\numobs(\state) + \ErrEst
  \Big] \ni \Vstar(\state) \qquad \mbox{uniformly for all $\state
  \in \StateSpace$}
\end{align}
with probability at least $1 - 3 \pardelta$.

Note that the dominant term in the error estimate $\ErrEst$ is
proportional to $\numobs^{-1/2} \diagnorm{\VarEst(\valuehat_\numobs,
  \Vdata_{\PEnumobs, \holdnumobs})}^{1/2}$, and it corresponds to an
estimate of the dominant term $\numobs^{-1/2}
\diagnorm{\optCovMat{\valuestar}}^{1/2}$ from the
bound~\eqref{eqn:PolEval-Algo-Cond}.  The second
bound~\eqref{eqn:general-CI-eqn2} in
Theorem~\ref{thm:main-thm-poleval} ensures that $\ErrEst$ provides an
optimal approximation, up to constant pre-factors, of the leading
order term on the right-hand side of the
bound~\eqref{eqn:PolEval-Algo-Cond}.

\subsubsection{Constructing the empirical covariance estimate}
\label{SecEmpCov}

We now provide details on the construction of the empirical covariance
estimate $\covarhat(\Vhat_\numobs; \MyData{})$.  Recalling the
definition~\eqref{eqn:Opt-cov-matrix}, our procedure is guided by the
decomposition
\begin{align*}
  \optCovMat{\valuestar} = \underbrace{(\Id - \contractPar
    \TranMat)^{-1}}_{\substack{\text{estimated using first} \\ \text{
        small holdout dataset}}} \;\;\;
  \underbrace{\valueCovOpt}_{\substack{\text{estimated using}
      \\ \text{remaining data}}} \;\;\; \underbrace{(\Id -
    \contractPar \TranMat)^{-\top}}_{\substack{\text{estimated using
        second} \\ \text{small holdout dataset}}}.
\end{align*}
From the given data $\MyData{} = (\MyData{\numobs}, \MyData{2
  \nhold})$, we partition $\MyData{2 \nhold}$ into two subsets
$(\MyDataLeft{\nhold}, \MyDataRight{\nhold})$, as indexed by
$\IndLeft$ and $\IndRight$ respectively.  We use this data to compute
the empirical averages
\begin{align}
\label{eqn:holdout-samples}
  \holdobsmat = \frac{1}{\holdnumobs} \sum_{i \in \holdoutsample}
  \obsmat_i \qquad \text{and} \qquad
  \secholdobsmat = \frac{1}{\holdnumobs} \sum_{i \in \secholdsample}
  \obsmat_i.
\end{align}

Using the dataset $\MyData{\numobs}$, we use the procedure $\AlgoEval$
to compute the value function estimate $\Vhat_\numobs$.  We re-use the
samples $\MyData{\numobs} = \{(\NoisyReward_j, \obsmat_j) \}_{j \in
  \varestsample}$ to compute the $V$-statistic
\begin{align}
\label{eqn:ValueCovEstimate}
\covarhat(\Vhat_\numobs; \MyData{\numobs}) & \defn
\frac{1}{\PEnumobs(\PEnumobs - 1)} \sum_{ \substack{j, k \in
    \varestsample \\ j < k}} (\NoisyReward_j - \NoisyReward_k +
\contractPar(\obsmat_j - \obsmat_k) \Vhat_\numobs)(\NoisyReward_j -
\NoisyReward_k + \contractPar(\obsmat_j - \obsmat_k)
\Vhat_\numobs)^\top.
\end{align}
With these three ingredients, our empirical covariance estimate is
given by
\begin{align}
  \label{eqn:final-estimate}
  \VarEst(\valuehat_\numobs, \MyData{}) & = (\Id -
  \contractPar\holdobsmat)^{-1} \; \covarhat(\valuehat_\numobs,
  \MyData{\numobs}) \; (\Id - \contractPar\secholdobsmat)^{-\top}.
\end{align}
Recalling that $\nhold \defn 2 \lceil \frac{32 \cdot
  \log(8\numstates^2 / \pardelta)}{(1 - \contractPar)^2} \rceil$, we
note that the hold-out sets will typically be very small relative to
the sample size $\numobs$ used to compute the estimate $\Vhat_\numobs$
itself.


\subsection{Early stopping with optimality: ~\ref{EmpIRE} procedure}

In practice, it is often the case that policy evaluation is carried
out repeatedly, as an inner loop within some broader algorithm (e.g.,
policy iteration schemes or actor-critic methods).  Due to repeated
calls to a policy evaluation routine, it is especially desirable in
these settings to minimize the amount of data used.  In practice, it
is adequate to terminate a policy evaluation routine once some target
accuracy $\error$ is achieved; guarantees in the $\ell_\infty$-norm,
such as those analyzed here, are particularly attractive in this
context.

In this section, we leverage the data-dependent bound from
Theorem~\ref{thm:main-thm-poleval} so as to propose \emph{early
stopping procedure} for policy evaluation, applicable for any
$(\specfast, \specslow)$-procedure for policy evaluation.  We refer to
it as the~\ref{EmpIRE} procedure, as a shorthand for
\textbf{Emp}irical \textbf{I}nstance-optimal Markov \textbf{R}eward
\textbf{E}valuation.  Taking as input a desired error level $\error$
and a tolerance probability $\deltaTol$, it is a sequential procedure
that operates over a sequence of epochs, where the amount of data
available increases dyadically as the epochs increase.  Upon
termination, it returns an estimate $\Vhat$ such that $\|\Vhat -
\Vstar\|_\infty \leq \error$ with probability at least $1 -
\deltaTol$.  Moreover, our theory guarantees that for any particular
instance, the number of samples used to do so is within a constant
factor of the minimum number of samples required for that particular
instance.  In this sense, the procedure is \emph{adaptive} to the
easiness or hardness of the instance at hand. \\

\begin{varalgorithm}{EmpIRE}
\caption{\textbf{Emp}irical \textbf{I}nstance-optimal Markov
  \textbf{R}eward \textbf{E}valuation}
\label{EmpIRE}
\begin{algorithmic}[1]
\STATE \texttt{Inputs:} (i) instance-optimal procedure $\AlgoEval$;
(ii) target accuracy $\error$ and (iii) tolerance probability
$\deltaTol$

\vspace{6pt}

\STATE Initialize $\PEbatch_0 = \frac{32}{(1 - \contractPar)^2}$,
$\pardelta_0 = \frac{\pardelta_{\textrm{target}}}{3}$, and
$\Vdata_{\holdoutsample}, \Vdata_{\varestsample}, \text{ and
}\Vdata_{\secholdsample}$ as empty sets.

\FOR {$\epoch = 1, \ldots$} \STATE Set tolerance parameter
$\pardelta_\epoch = \frac{\pardelta_0}{2^\epoch}$, along with
\begin{align*}
\mbox{Holdout size} \quad \doublehold{\epoch} \defn \PEbatch_0 \cdot
\log(4 \numstates^2 / \pardelta_\epoch), \quad \text{and batch size}
\quad \PEbatch_\epoch = 2^{\epoch} \PEbatch_0 \cdot
\specfast^2(\pardelta_\epoch) \cdot \log(4 \numstates /
\pardelta_\epoch).
\end{align*}
\STATE Augment data sets $\Vdata_{\holdoutsample},
\Vdata_{\varestsample}, \text{ and }\Vdata_{\secholdsample}$ with
additional i.i.d. samples such that

$|\holdoutsample| = |\secholdsample| = \doublehold{\epoch}$, and
$|\varestsample| = \PEbatch_\epoch$. Set $\Vdata =
\Vdata_{\holdoutsample} \cup \Vdata_{\varestsample} \cup
\Vdata_{\secholdsample}$.
  \STATE Compute estimate $$\valuehat \leftarrow \AlgoEval
  \big(\MyData{\varestsample} \big). $$
  \STATE Evaluate empirical error estimates
  \begin{align}
    \vfast \defn \frac{2\sqrt{6} \cdot \specfast(\pardelta_\epoch)
    }{\sqrt{\PEbatch_\epoch}} \cdot \VarEst(\valuehat, \MyData{}) ,
    \quad \mbox{and} \quad \vslow \defn \frac{2
      \specslow(\pardelta_\epoch)}{\PEbatch_\epoch} + \frac{6
      \bfun(\valuehat)}{1 - \contractPar} \cdot \frac{\sqrt{\log(8
        \numstates / \pardelta_\epoch)}}{\PEbatch_\epoch - 1}.
  \end{align}
  \IF {$\vfast + \vslow < \error$} \STATE Terminate \ENDIF \ENDFOR
\end{algorithmic}
\end{varalgorithm}

\noindent The following corollary, which is a simplified version of Corollary~\ref{cors:poleval-stopping-guarantee2}, guarantees that the~\ref{EmpIRE}
procedure terminates with high probability, and returns an
$\error$-accurate estimate of the value function while using a total
sample size that is well-controlled.
For ease of presentation, in Corollary~\ref{cors:poleval-stopping-guarantee} we assume that
\begin{align}
\label{eqn:fast-and-slow-factors-relation}
  \tfrac{\specslow(\delta)}{\specfast^2(\delta)} \leq c_0 \quad \text{for all} 
 \quad \delta > 0. 
\end{align}
where, $c_0$ is a universal constant. This assumption, although
satisfied by both the instance-optimal
algorithms~\cite{khamaru2020PE,MKWBJVariance22} discussed in the
paper, is not necessary; we impose this condition only to provide a
simpler upper bound on the number of epochs and number of samples used
by the protocol~\ref{EmpIRE}.  A version of
Corollary~\ref{cors:poleval-stopping-guarantee} without the
assumption~\eqref{eqn:fast-and-slow-factors-relation} can be be found
in Corollary~\ref{cors:poleval-stopping-guarantee2}.

With the assumption~\eqref{eqn:fast-and-slow-factors-relation} at
hand, and given a tolerance probability $\deltaTol \in (0,1)$ we
define
\begin{multline}
\label{eqn:Epoch-max}
\epochs_{\max} \defn \log_2 \max \left\{\tfrac{ (1 - \contractPar)^2
  \, \diagnorm{\optCovMat{\valuestar}}}{\error^2}, \; \tfrac{c_0(1 -
  \contractPar)^2}{4\error} + \tfrac{(1 - \contractPar)
  \bfun(\valuestar)}{\error} \cdot \sqrt{\log(\tfrac{8
    \numstates}{\deltaTol})} \right\}.
\end{multline}
In the next statement, $(c_1, c_2)$ are universal positive constants.
\begin{cors}
\label{cors:poleval-stopping-guarantee}
Given any algorithm~$\AlgoEval$ satisfying
condition~\eqref{eqn:PolEval-Algo-Cond}, suppose that we run
the~\ref{EmpIRE} protocol with target accuracy $\error$ and tolerance
probability $\deltaTol$, and assume that 
condition~\eqref{eqn:fast-and-slow-factors-relation} is in force. Then with probability $1 - \deltaTol$:
\begin{enumerate}
\item[(a)] It terminates after a number of epochs $\epochs$ bounded as
  $\epochs \leq c_1 + \epochs_{\max}$.
\item[(b)] Upon termination, it returns an estimate $\Vhat$ such
  $\|\Vhat - \Vstar\|_\infty \leq \error$.
\item[(c)] The total number of samples used satisfies the bound
\begin{align}
\label{eqn:corollary-CI-eqn3}
\PEnumobs \leq \plaincon_2 \max \left\{\frac{
  \specfast^2(\pardelta_\epochs)}{\error^2} \cdot
\diagnorm{\optCovMat{\valuestar}}, \frac{1}{\error} \left[
  \specslow(2\pardelta_\epochs) + \frac{\bfun(\valuestar)}{1 -
    \contractPar} \sqrt{\log(4\numstates / \pardelta_\epochs)}
  \right] \right\}
\end{align}
where $\pardelta_\epochs = \tfrac{\deltaTol}{2^\epochs}$.
\end{enumerate}
\end{cors}
\noindent See Appendix~\ref{sec:CI-cor-proof} for the the proof.\\

Let us make a few comments about this claim; in this informal
discussion, we omit constants and terms that are logarithmic in the
pair $(\numstates, 1/\delta)$.
Corollary~\ref{cors:poleval-stopping-guarantee} guarantees that, with
a controlled probability, the algorithm terminates and returns an
accurate answer.  Moreover, over all epochs, it uses of the order
$\bigOh\left( \frac{\diagnorm{\optCovMat{\valuestar}}}{\error^2}
\right)$ samples to achieve $\error$ accuracy in the
$\ell_\infty$-norm.  This guarantee should be compared to local
minimax lower bounds from our past work~\cite{khamaru2020PE}: any
procedure with an oracle that has access to the true error
$\inftynorm{\valuehat - \valuestar}$ requires at least
$\frac{\diagnorm{\optCovMat{\valuestar}}}{\error^2}$ samples. Thus,
our procedure utilizes the same number of samples as any
instance-optimal procedure equipped with such an oracle, up to
constant and logarithmic factors.

Next, observe that the bound in the right hand side of the
expression~\eqref{eqn:Epoch-max} depends on $\log(1/\epsilon)$, and is
of higher order. Ignoring it for the moment, the number of epochs
required to achieve a given target accuracy $\error > 0$ is of the
order $\log_2\left(\frac{(1 - \contractPar)^2}{\error^2} \cdot
\diagnorm{\optCovMat{\valuestar}} \right)$ epochs, to an additive
constant. Finally, we refer the reader to
Corollary~\ref{cors:poleval-stopping-guarantee2} in
Appendix~\ref{sec:CI-cor-proof} for an analogue
Corollary~\ref{cors:poleval-stopping-guarantee} without the
assumption~\eqref{eqn:fast-and-slow-factors-relation}.


\subsection{Some numerical simulations}

In this section, we use a simple two stateMarkov reward process (MRP)
$\MRP = (\TranMat, \reward)$ to illustrate the behavior of our
methods; this family of instances was introduced in past work by a
subset of the current authors~\cite{pananjady2021instance,
  khamaru2020PE}. For a parameter $\prob \in [0,1 ]$, consider a
transition matrix $\TranMat \in \RR^{2 \times 2}$ and reward vector
$\reward \in \RR^2$ given by
\begin{align*}
\TranMat = \begin{bmatrix} \prob & 1 - \prob \\ 0 & 1 \end{bmatrix},
\quad \text{and} \quad \reward = \begin{bmatrix} 1
  \\ \rewscale \end{bmatrix}.
\end{align*}
where $\rewscale \in \real$ along with the discount $\contractPar \in
[0, 1)$ are additional parameters of the construction. Fix a scalar
  $\diff \geq 0$, and then set
\begin{align*}
\prob = \frac{4 \contractPar - 1}{3 \contractPar}, \quad \text{and}
\quad \rewscale = 1 - (1 - \contractPar)^{\diff}.
\end{align*}
Note that this is the MRP induced via the MDP defined in
Example~\ref{ExaSimple} given by the policy $\policy(\state) =
\action_1$ for $\state = \state_1, \state_2$.

By calculations similar to those in
Appendix~\ref{sec:ExaSimpleCalulations}, we have
\begin{align*}
\diagnorm{\optCovMat{\valuestar}}^{1/2} \leq c \cdot \left(\frac{1}{1
  - \contractPar}\right)^{1.5 - \diff}.
\end{align*}
Thus for larger choices of $\contractPar$ we see a ``harder'' problem,
and for larger $\diff$ we see an ``easier'' problem, indicated by the
complexity functional. If we knew
$\diagnorm{\optCovMat{\valuestar}}^{1/2} = c \cdot \left(\frac{1}{1 -
  \contractPar}\right)^{1.5 - \diff}$, then we could easily see that
an instance-optimal algorithm requires $\PEnumobs \gtrsim
\frac{1}{\error^2}\left(\frac{1}{1 - \contractPar} \right)^{3 -
  2\diff}$ samples to achieve $\error$ accuracy in the
$\ell_\infty$-norm. Typically we do not have knowledge of
$\diagnorm{\optCovMat{\valuestar}}^{1/2}$ prior to running the
algorithm so we would need to use at least $\PEnumobs \gtrsim
\frac{1}{\error^2(1 - \contractPar)^3}$ samples to guarantee $\error$
accuracy. For $\diff \geq 0.5$, it becomes apparent that an
instance-optimal algorithm requires fewer samples than the worst-case
guarantees indicate, and the differences become drastic for larger
$\diff$ and $\contractPar$. The~\ref{EmpIRE} algorithm allows the user
to exploit the instance-specific difficulty of the problem at hand and
avoid the worst-case scenario of requiring $\PEnumobs \gtrsim
\frac{1}{\error^2(1 - \contractPar)^3}$ samples.

We describe the details of the numerical simulations ran for the
policy evaluation setting here. For every combination of
$(\contractPar, \diff)$ we ran Algorithm~\ref{EmpIRE} with the ROOT-SA
algorithm~\cite{MKWBJVariance22} as our instance-optimal sub-procedure
on the MRP for $1000$ trials, measuring the factor savings, i.e. the
ratio of the number of samples required in the worst case to the
number of samples actually used, as well as for our estimate
$\valuehat$ the final predicted error for our estimated and the true
error. The $\contractPar$'s were chosen to be uniformly spaced between
$0.9$ and $0.99$ in the log-scale, and $\diff$ was chosen to be in the
set $\{1.0, 1.5\}$. The desired tolerance was chosen to be $\error =
0.1$. Our results are presented in Figure~\ref{FigFront}, as
previously described. The initial point $\values_0$ was chosen by
setting aside $\frac{2}{(1 - \contractPar)^2}$ samples to construct a
plug-in estimate of $\valuestar$. As expected, increasing both
$\contractPar$ and $\diff$ increases the savings in the number of
samples used, as compared to the worst-case
guarantees. Figure~\ref{FigFront} highlights the improvement in sample
size requirements of Algorithm~\ref{EmpIRE}, as compared to the using
the worst-case guarantees and illustrates the benefits of exploiting
the local structure of the problem at hand.

Figure~\ref{plot_errors} ensures to verify that our theoretical
guarantees describe the behavior observed in practice. For every
combination of $(\contractPar, \diff)$ we run $1000$ trials of
Algorithm~\ref{EmpIRE} and keeping track of the predicted error given
by equation~\eqref{eqn:general-CI-eqn1} as well as the true error
$\inftynorm{\valuehat - \valuestar}$. Algorithm~\ref{EmpIRE} was run
with chosen tolerance $\error = 0.1$. Our theory ensures that the true
error should be consistently below the predicted error for all
combinations of $\contractPar$ and $\diff$, which is the behavior
illustrated in Figure~\ref{plot_errors}. The plots also illustrate
that the true error is consistently far below the predicted error
(which itself is consistently below the specified tolerance $\error$),
demonstrating that our predictions are relatively conservative, and
that higher-order terms can potentially be dropped in error estimates
while still remaining a viable algorithm. Overall,
Figure~\ref{plot_errors} illustrates that the bounds on
Algorithm~\ref{EmpIRE} are correct and highlights its practical
utility in an idealized setting.
\begin{figure}[ht!]
\begin{center}
\begin{tabular}{ccc}
  \widgraph{0.45\textwidth}{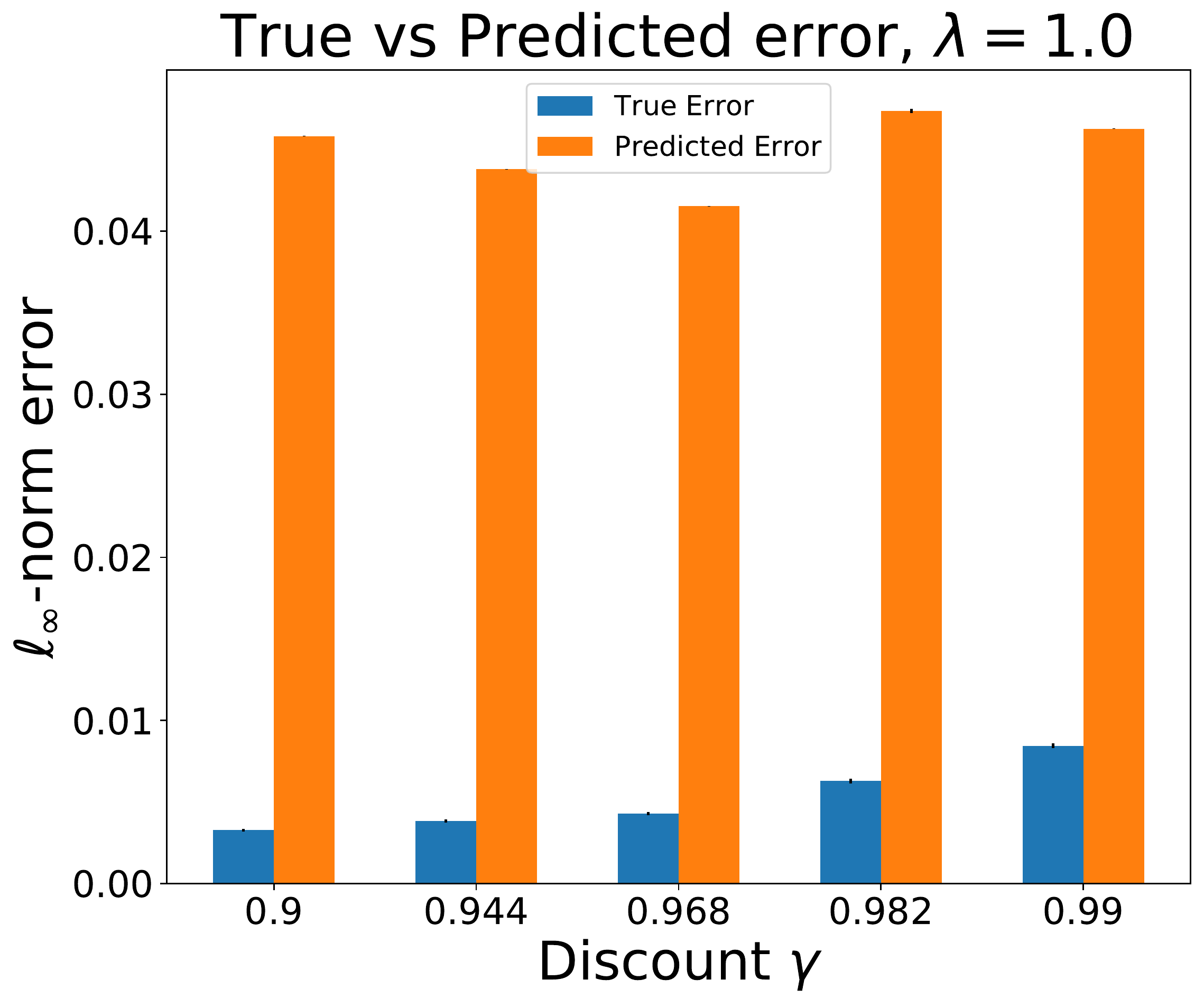} &&
  \widgraph{0.45\textwidth}{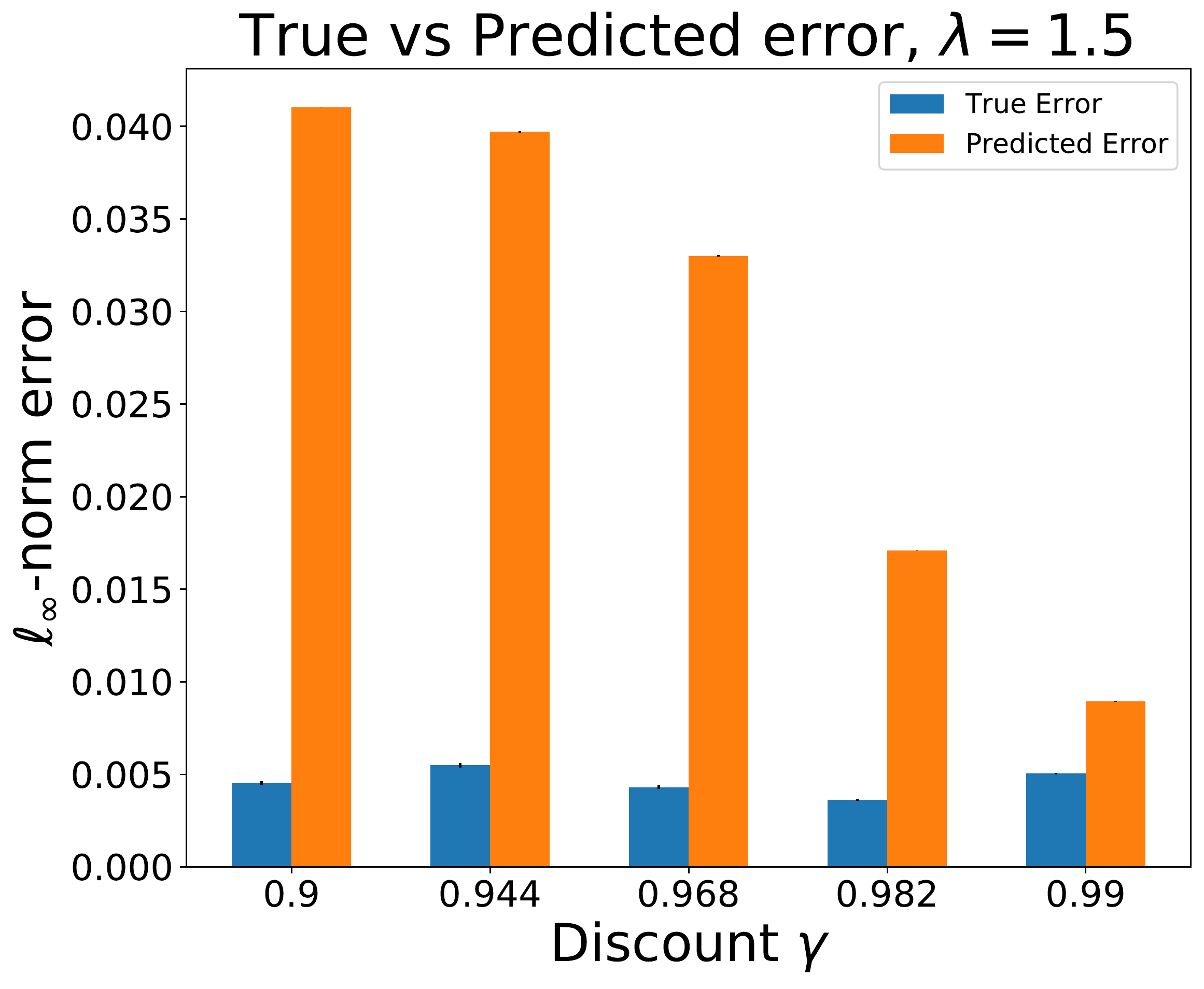}
  \\ (a) && (b)
  \end{tabular}
  \caption{Illustration of the termination behavior of
    Algorithm~\ref{EmpIRE} applied on the MRP. Plots the average of
    the true error (blue) and predicted error (orange) along with
    error bars denoting the standard deviation for different choices
    of $\contractPar$ and for (a) $\diff = 1.0$ and (b) $\diff =
    1.5$.}
\label{plot_errors}
\end{center}
\end{figure}


\section{Confidence regions for optimal $Q$-function estimation}
\label{SecPolOpt}

In this section, we derive confidence regions for optimal $Q$-function
estimation problem.  Let us first describe the functional that
characterizes the difficulty of the optimal $Q$-function estimation
problem.

\subsection{Instance-dependence for optimal $Q$-functions}

Given a sample $(\obsmatQ, \NoisyRewardQ)$ from our observation model
(see equation~\eqref{EqnGenerativeMDP}), we can define the
single-sample empirical Bellman operator as
\begin{align}
\label{EqnEmpiricalBellmanOpt}
\NoisyOptOp(\Qvalues) & \defn \NoisyRewardQ(\state, \action) +
\contractPar \sum_{\state' \in \stateset} \obsmatQ_\action(\state'
\mid \state) \max_{\action' \in \actionset} \Qvalues(\state',
\action'),
\end{align}
where we have introduced $\obsmatQ_\action(\state' \mid \state) \defn
\Indicator_{ \obsmatQ(\state, \action) = \state'}$.  Throughout this
section, we use the shorthand $\Dim = \card{\stateset \times
  \actionset}$.

In a recent paper~\cite{khamaru2021instance}, the authors
show the quantity that determines the difficulty of estimating the
optimal value function $\Qvaluestar$ is\footnote{The
quantity~$\QCovInstOpt(\rewardQ, \TranMatQ, \contractPar)$ is a the
same of the term~$\max_{\policystar \in \OptPolSet}
\inftynorm{\nu(\policystar; \TranMatQ, \rewardQ, \contractPar)}$ from
Theorem~1 from the paper~\cite{khamaru2021instance}.}
\begin{align}
\label{eqn:PolOptLB}
  \QCovInstOpt(\rewardQ, \TranMatQ, \contractPar) \defn \max_{\policy
    \in \OptPolSet} \; \diagnorm{(\Id - \contractPar
    \TranMatQ^\policy)^{-1} \optCovMatQ{\Qvaluestar} (\Id -
    \contractPar \TranMatQ^\policy)^{-\top}}^{\frac{1}{2}},
\end{align}
where the quality of an estimate is being measured by its
$\ell_\infty$ distance from $\Qvaluestar$. Here the set $\OptPolSet$
denotes the set of all optimal policies, the matrix
$\optCovMatQ{\Qvaluestar} \defn \covar(\NoisyOptOp(\Qvaluestar))$, and
$\TranMatQ^{\policy}$ is a right-linear mapping of $\RR^{\Dim}$ to
itself, whose action on a $Q$-function is given by
\begin{align*}
  \TranMatQ^\policy(\Qvalues)(\state, \action) = \sum_{\state' \in
    \stateset} \TranMatQ_{\action}(\state' \mid \state) \cdot
  \Qvalues(\state', \policy(\state')) \qquad \text{for all} \quad
  (\state, \action) \in \stateset \times \actionset.
\end{align*}
Note that by construction, the quantity $\NoisyOptOp(\Qvaluestar)$ is
an unbiased estimate of $\BellOptOp(\Qvaluestar)$, and the covariance
matrix $\covar(\NoisyOptOp(\Qvaluestar))$ in
equation~\eqref{eqn:PolOptLB} captures the noise present in the
empirical Bellman operator~\eqref{EqnEmpiricalBellmanOpt} as an
estimate of the population Bellman
operator~\eqref{EqnPopulationBellmanOpt}, when evaluated at the
optimal value function $\Qvaluestar$. As for the pre-factor $(\Id -
\contractPar \TranMatQ^{\pi})^{-1}$, by a von Neumann expansion we can
write
\begin{align*}
(\Id - \contractPar \TranMatQ^{\pi})^{-1} & = \sum_{k=0}^\infty
  (\contractPar \TranMatQ^\pi)^k.
\end{align*}
The sum of the powers of $\contractPar \TranMatQ^\policy$ accounts for
the compounded effect of an initial perturbation when following the
Markov chain specified by an optimal policy $\policy$. Put simply, the
quantity~\eqref{eqn:PolOptLB} captures the noise accumulated when an
optimal policy $\pi$ is followed.

The authors in the paper~\cite{khamaru2021instance} also show that
under appropriate assumptions, the estimate $\Qvaluehat_\numobs$,
obtained from a variance reduced $Q$-learning algorithm using
$\numobs$ i.i.d. samples, satisfies the bound
\begin{align}
\label{eqn:VR-Q-guarantee}
  \inftynorm{\Qvaluehat_\numobs - \Qvaluestar} \leq
  \specfast(\pardelta) \cdot \frac{\QCovInstOpt(\rewardQ, \TranMatQ,
    \contractPar)}{\sqrt{\numobs}} + \frac{c_2}{\numobs}
\end{align}
with probability at least $1 - \pardelta$. Here the functions
$\specfast$ and $\specslow$ depend on the tolerance parameter
$\pardelta \in (0,1)$, along with logarithmic factors in the
dimension.  Furthermore, this algorithm achieves the non-asymptotic
local minimax lower bound for the optimal $Q$-function estimation
problem (see Theorem 1 in the
paper~\cite{khamaru2021instance}).

\subsection{A conservative yet useful upper bound}
\label{sec:Convervative-UB-PolOpt}
Motivated by the success in Section~\ref{SecPolEval}, it is
interesting to ask if we can prove a data-dependent estimate for the
term $\QCovInstOpt(\rewardQ, \TranMatQ, \contractPar)$. Observe that
the term $\QCovInstOpt(\rewardQ, \TranMatQ, \contractPar)$ from
equation~\eqref{eqn:PolOptLB} depends on the set of \emph{all} optimal
policies $\OptPolSet$, and the authors are not aware of a
data-dependent efficient estimate without imposing restrictive
assumptions on the MDP~$\MDP = (\rewardQ, \TranMatQ,\contractPar)$.

Interestingly, simple algebra reveals that
\begin{align}
  \label{eqn:PolOptLooseUB}
  \QCovInstOpt(\rewardQ, \TranMatQ, \contractPar) \leq
  \frac{\diagnorm{\optCovMatQ{\Qvaluestar}}^{\frac{1}{2}}}{1 -
    \contractPar}.
\end{align}
Observe that all entries of the PSD matrix $\optCovMatQ{\Qvaluestar}$
are upper bounded by the scalar $\diagnorm{\optCovMatQ{\Qvaluestar}}$,
the entries of the matrix $(\Id - \contractPar
\TranMatQ^\policy)^{-1}$ are non-negative, and \mbox{$\|(\Id -
  \contractPar \TranMatQ^\policy)^{-1}\|_{1, \infty} \leq \frac{1}{1 -
    \contractPar}$}. Combining these three observations yield the
claim~\eqref{eqn:PolOptLooseUB}.

It is interesting to compare the upper bound~\eqref{eqn:PolOptLooseUB}
with a worst case upper bound on $\QCovInstOpt(\rewardQ, \TranMatQ,
\contractPar)$. In light of Lemma~7 from the
paper~\cite{Azar2013Minimax}, assuming $\rewardboundQ \leq 1$, and
$\inftynorm{\rewardQ} \leq 1$ for simplicity, we have
\begin{align}
\label{eqn:PlOptLooseUB-2}
  \QCovInstOpt(\rewardQ, \TranMatQ, \contractPar) \leq \frac{1}{(1 -
    \contractPar)^{1.5}}.
\end{align}
Combining the bounds~\eqref{eqn:PolOptLooseUB}
and~\eqref{eqn:PlOptLooseUB-2}, we see that the upper
bound~\eqref{eqn:PolOptLooseUB} is particularly useful when
\begin{align*}
  \diagnorm{\optCovMatQ{\Qvaluestar}}^{\frac{1}{2}} \ll
  \frac{1}{\sqrt{1 - \contractPar}}.
\end{align*}
As an illustration, we now describe an interesting sequence of MDPs
for which the last condition is satisfied.

\begin{example}[A continuum of illustrative examples]
  \label{ExaSimple}
Consider an MDP with two states $\{\state_1, \state_2\}$, two actions
$\{\action_1, \action_2\}$, and with transition functions and reward
functions given by
\begin{align}
\label{eqn:ex-1-transitions-andrewards}
    \TranMatQ_{\action_1} = \begin{bmatrix} p & 1 - p \\ 0 &
      1 \end{bmatrix} \quad \TranMatQ_{\action_2} = \begin{bmatrix} 1
      & 0 \\ 0 & 1 \end{bmatrix}, \quad \mbox{and} \quad \rewardQ
    = \begin{bmatrix} 1 & 0 \\ \tau & 0 \end{bmatrix}.
\end{align}
We assume that there is no randomness in the reward samples. Here the
pair $(p, \taupar)$ along with the discount factor $\contractPar$ are
parameters of the construction, and we consider a sub-family of these
parameters indexed by a scalar $\lampar \geq 0$.  For any such
$\lampar$ and discount factor $\contractPar \in (\tfrac{1}{4}, 1)$,
consider the setting
\begin{align*}
    p = \frac{4 \contractPar - 1}{3 \contractPar}, \quad \text{and}
    \quad \taupar = 1 - (1 - \contractPar)^{\lampar}.
\end{align*}
With these choices of parameters, the optimal $Q$-function $\Qstar$
takes the form
\begin{align*}
\Qstar = \begin{bmatrix} \frac{(1 - \contractPar) + \contractPar
    \taupar (1 - p)}{(1 - \contractPar)(1 - \contractPar p)} &
  \contractPar \cdot \frac{(1 - \contractPar) + \contractPar \taupar
    (1 - p)}{(1 - \contractPar)(1 - \contractPar p)}
  \\ \frac{\taupar}{1 - \contractPar} & \frac{\contractPar \taupar}{1
    - \contractPar} \end{bmatrix},
\end{align*} 
with an unique optimal policy $\policystar(\state_1) =
\policystar(\state_2) = \action_1$.  We can then compute that
\begin{align}
\label{eqn:simple-instance-bound}
  \diagnorm{\optCovMatQ{\Qvaluestar}}^{\frac{1}{2}} = \frac{1 -
    \taupar}{1 - \contractPar p} \cdot \sqrt{p(1 - p)} = c \cdot
  \left(\frac{1}{1-\contractPar}\right)^{\frac{1}{2} - \lampar} <
  \frac{1}{\sqrt{1 - \contractPar}}.
\end{align}
See Appendix~\ref{sec:ExaSimpleCalulations} for the details of this
calculation.

Example~\ref{ExaSimple} shows that for any $\lampar > 0$, the bound
$\frac{\diagnorm{\optCovMatQ{\Qvaluestar}}^{\frac{1}{2}}}{1-
  \contractPar}$ is smaller than the worst case bound of $\frac{1}{(1
  - \contractPar)^{1.5}}$ by a factor of $\frac{1}{(1 -
  \contractPar)^\lampar}$.  We point out that this gap is significant
when the discount factor $\contractPar$ is close to $1$. For instance,
for $\contractPar = 0.99$ and $\lampar = 1$, the upper
bound~\eqref{eqn:PolOptLooseUB} is $10^2$ times better than the worst
case bound. Alternatively, the bound~\eqref{eqn:VR-Q-guarantee} yields
an improvement of a factor of $10^4$, when compared to a worst case
value.
\end{example}

\subsection{Data-dependent bounds for optimal $Q$-functions}
\label{sec:Data-dependent-bound}

In this section, we provide a data-dependent upper bound on
$\inftynorm{\widehat{Q} - \Qstar}$, where the estimate
$\widehat{\Qvalues}$ is obtained from any algorithm $\AlgoPolOpt$
satisfying certain convergence criterion.

\subsubsection{Instance-valid algorithms}
We assume that given access to $\numobs$ generative samples
\mbox{$\Qdata_\numobs \defn \{ \NoisyRewardQ_i, \obsmatQ_i \}_{i =
    1}^\numobs$} from an MDP~$\MDP = (\rewardQ, \TranMatQ,
\contractPar)$, the output $\Qvaluehat_\numobs \defn
\AlgoPolOpt(\Qdata_\numobs)$, obtained using the data set
$\Qdata_\numobs$, satisfies
\begin{align}
\label{eqn:PolOpt-algo-prop}
\AlgoPolOpt \texttt{ condition:} \qquad \inftynorm{\Qvaluehat_\numobs
  - \Qvaluestar} \leq \frac{\specfast(\pardelta)}{\sqrt{\numobs}}
\cdot \frac{\diagnorm{\optCovMatQ{\Qvaluestar}}^{\frac{1}{2}}}{(1 -
  \gamma)} + \frac{\specslow(\pardelta)}{\numobs}.
\end{align}
with probability $1 - \pardelta$.  Here $\specfast$ and $\specslow$
are functions of the tolerance level $\pardelta$; in practical
settings, they may in addition depend on the problem dimension in some
cases, but we suppress this dependence for simplicity.

We point out that any instance-optimal algorithm, i.e. algorithms
satisfying the condition~\eqref{eqn:VR-Q-guarantee}, automatically
satisfies the condition~\eqref{eqn:PolOpt-algo-prop}. One added
benefit of the condition~\eqref{eqn:PolOpt-algo-prop} is that this
condition is much easier to verify. To illustrate this point, in
Proposition~\ref{prop:vrqlearn-bound} in
Appendix~\ref{sec:VRQ-prop-proof} , we show that the variance reduced
$Q$-learning algorithm from the paper~\cite{khamaru2021instance}
satisfies the condition~\eqref{eqn:PolOpt-algo-prop} under a
\emph{milder assumption} (cf. Proposition~\ref{prop:vrqlearn-bound} in
Appendix~\ref{sec:VRQ-prop-proof} and Theorem 2 from the
paper~\cite{khamaru2021instance} ).  In the rest of the section, we
focus on providing a data-dependent estimate for
$\diagnorm{\optCovMatQ{\qvaluesstar}}^{\frac{1}{2}}$.

\subsubsection{Constructing a data-dependent bound}

In order to obtain a data-dependent bound on
$\inftynorm{\Qvaluehat_\numobs - \Qvaluestar}$, it suffices to
estimate the complexity term
$\diagnorm{\optCovMatQ{\Qvaluestar}}^{1/2}$. Our next result
guarantees that there is a data-dependent estimate
$\QVarEst(\Qvaluehat_\numobs, \Qdata)$ which provides an upper bound
that holds with high probability, and is within constant factors of
$\diagnorm{\optCovMatQ{\Qvaluestar}}^{1/2}$. With a recycling of
notation, we define \mbox{$\bfun(\Qvalues) \defn \rewardbound +
  \contractPar \inftynorm{\Qvalues}$,} and construct the
\textit{empirical error estimate}
\begin{align}
\label{EqnQempError}  
\ErrEstQ(\Qvaluehat_\Qnumobs, \Qdata_\Qnumobs, \pardelta) \defn
\frac{2\sqrt{2} \cdot \specfast(\pardelta)}{\sqrt{\numobs}} \cdot
\frac{\diagnorm{\QVarEst(\Qvaluehat_\numobs,
    \Qdata_\numobs)}^{1/2}}{(1 - \contractPar)} +
\frac{2\specslow(\pardelta)}{\Qnumobs} + \frac{8
  \bfun(\Qvaluehat_\Qnumobs)}{1 - \contractPar} \cdot
\frac{\specfast(\pardelta) \sqrt{2 \log(\Dims /\pardelta)}}{\Qnumobs -
  1}.
\end{align}
In this error estimate, the base estimator uses the $\Qnumobs$-sample
dataset $\Qdata_\Qnumobs$ both to compute the $Q$-value function
estimate $\Qvaluehat_\Qnumobs \defn \AlgoPolOpt(\Qdata_\numobs)$.  It
also uses the same data to estimate the covariance
$\QVarEst(\Qvaluehat_\Qnumobs, \Qdata_\Qnumobs)$, according to the
procedure given in Section~\ref{SecProcedure}. \\

\noindent The following result summarizes the guarantees associated
with the empirical error estimate~\eqref{EqnQempError}:
\begin{theorem}
\label{thm:general-CI-Policy-opt}
Let $\AlgoPolOpt$ be any $(\specfast, \specslow)$ instance-dependent
algorithm (cf. relation~\eqref{eqn:PolOpt-algo-prop}). Then for any
$\pardelta \in (0, 1)$, given a dataset $\Qdata_\Qnumobs$ with
$\Qnumobs \geq \specfast^2(\pardelta) \cdot \frac{32 \cdot \log(4
  \Dims / \pardelta)}{(1 - \contractPar)^2}$ samples, the following
holds with probability at least $1 - 2\pardelta$:
\begin{enumerate}
\item[(a)] The $\ell_\infty$-error is upper bounded as
\begin{subequations}
\begin{align}
\label{eqn:general-CI-PolOpt-1}
\inftynorm{\Qvaluehat_\Qnumobs - \Qvaluestar} \leq
\ErrEstQ(\Qvaluehat_\Qnumobs, \Qdata_\Qnumobs, \pardelta).
\end{align}
\item[(b)] Moreover, this guarantee is order-optimal in the sense that
\begin{align}
\label{eqn:general-CI-PolOpt-2}
\ErrEstQ(\Qvaluehat_\Qnumobs, \Qdata_\Qnumobs, \pardelta) \leq \frac{7
  \cdot \specfast(\pardelta)}{\sqrt{\Qnumobs}} \cdot
\frac{\diagnorm{\optCovMatQ{\Qvaluestar}}^{1/2}}{1 - \contractPar} +
\frac{6 \specslow(\pardelta)}{\Qnumobs} + \frac{5
  \bfun(\Qvaluestar)}{1 - \contractPar} \cdot \frac{\sqrt{\log(\Dims /
    \pardelta)}}{\Qnumobs - 1}.
\end{align}
\end{subequations}
\end{enumerate}
\end{theorem}
\noindent See Section~\ref{sec:Proof-of-Policy-Opt-Thm} for the
proof.\\

A few comments regarding Theorem~\ref{thm:general-CI-Policy-opt} are
in order. Like in policy evaluation, the empirical error estimate
$\ErrEstQ$ can be computed based on the data and we obtain a
data-dependent confidence interval for $\Qvaluestar$. In particular,
equation~\eqref{eqn:general-CI-PolOpt-1} guarantees that
\begin{align*}
\left[ \Qvaluehat_\Qnumobs(\state, \action) - \ErrEstQ,
  \Qvaluehat_\Qnumobs(\state, \action) + \ErrEstQ \right] \ni
\Qvaluestar(\state, \action) \qquad \text{uniformly for all } \state
\in \stateset, \action \in \actionset
\end{align*}
with probability at least $1 - 2 \pardelta$.

We point out that the dominating term in the error estimate $\ErrEstQ$
is proportional to $\numobs^{-1/2}
\diagnorm{\QVarEst(\Qvaluehat_\numobs, \Qdata_\numobs)}^{1/2}$, and
this is an estimate of the dominating term $\numobs^{-1/2}
\diagnorm{\optCovMatQ{\Qvaluestar}}^{1/2}$ from the
bound~\eqref{eqn:PolOpt-algo-prop}.  Additionally, the
bound~\eqref{eqn:general-CI-PolOpt-2} ensures that the proposed
data-dependent bound in equation~\eqref{eqn:general-CI-PolOpt-1} is a
sharp approximation of $\numobs^{-1/2}
\diagnorm{\optCovMatQ{\Qvaluestar}}^{1/2}$ from the
bound~\eqref{eqn:PolOpt-algo-prop}.


\subsubsection{Estimation of $\diagnorm{\optCovMatQ{\qvaluesstar}}^{\frac{1}{2}}$}
\label{SecProcedure}

In this section, we provide the details on the construction of the
estimate $\diagnorm{\QVarEst(\Qvaluehat_\numobs,
  \Qdata_\numobs)}^{1/2}$ of
$\diagnorm{\optCovMatQ{\qvaluesstar}}^{\frac{1}{2}}$.  Given a
generative data set $\mbox{$\Qdata_\numobs \defn \{ \NoisyRewardQ_i,
  \obsmatQ_i \}_{i = 1}^\numobs$}$, we first use $\Qvaluehat_\numobs =
\AlgoPolOpt(\Qdata_\numobs)$ to estimate $\Qvaluestar$, and use the
same data set $\Qdata_\numobs$ to provide an empirical estimate of the
quantity $\diagnorm{\optCovMatQ{\Qvaluestar}}^{\frac{1}{2}}$. Observe
that it suffices to estimate the diagonal entries of the matrix
$\optCovMatQ{\Qvaluestar}$. Accordingly, our empirical covariance
matrix $\QVarEst( \Qvaluehat_\numobs, \Qdata_\numobs)$, based on a
data set \mbox{$\Qdata_\numobs$}, is a $\Dim \times \Dim$
\emph{diagonal matrix}. Concretely, let $\NoisyOptOp_i$ denote the
empirical estimate of the Bellman optimality operator based on the
$i^{th}$ generative sample, i.e.,
\begin{align*}
\NoisyOptOp_i(\Qvalues)(\state, \action) = \NoisyRewardQ_i(\state,
\action) + \contractPar \sum_{\state' \in \stateset}
\obsmatQ_i^{\action}(\state' \mid \state) \cdot \max_{\action \in
  \actionset} \Qvalues(\state', \action), \quad \text{for all} \;\;
(\state, \action) \in \stateset \times \actionset.
\end{align*}
Recall that $\obsmatQ_i^{\action}(\state' \mid \state)$ is $1$ if and
only if the sample in $\obsmatQ_i$ from $(\state, \action)$ gives
$\state'$ as the next state, and zero otherwise.  We define the
$(\state, \action)^{th}$ diagonal entry of the diagonal matrix
$\QVarEst( \Qvaluehat_\numobs, \Qdata_\numobs)$ as
\begin{align}
\label{eqn:QVarEst-defn}
\QVarEst( \Qvaluehat_\numobs, \Qdata_\numobs)((\state, \action),
(\state, \action)) = \frac{1}{\numobs(\numobs - 1)} \sum_{1 \leq i < j
  \leq k} \left\{ \NoisyOptOp_i(\Qvaluehat_\numobs)(\state, \action) -
\NoisyOptOp_j(\Qvalues)(\state, \action) \right\}^2.
\end{align}

\subsubsection{Early stopping for optimal $Q$-functions}

Inspired by Algorithm~\ref{EmpIRE}, we can take an instance-dependent
procedure $\AlgoPolOpt$ satisfying the
condition~\eqref{eqn:PolOpt-algo-prop} and construct a similar
stopping procedure that terminates when the data-dependent upper bound
of $\inftynorm{\qvalueshat - \qvaluesstar}$ estimate is below some
user-specified threshold $\error$. Following Algorithm~\ref{EmpIRE},
one would iteratively recompute the estimate $\qvalueshat$ using the
algorithm~$\AlgoPolOpt$, and check the condition
\begin{align*}
 & \frac{2\sqrt{2} \specfast(\pardelta_\numobs)}{\sqrt{\numobs}} \cdot
  \frac{\diagnorm{\QVarEst(\Qvaluehat_\numobs,
      \Qdata_\numobs)}^{1/2}}{(1 - \contractPar)} +\frac{2\sqrt{2}
    \specfast(\pardelta_\numobs) \cdot\sqrt{16\log(D/\pardelta)} }{(1
    - \contractPar)} \cdot
  \frac{\bfun(\Qvaluehat_\numobs)}{\numobs - 1} + \frac{2
    \specslow{\pardelta_\numobs}}{\numobs} \leq \error.
\end{align*}
We terminate the algorithm~$\AlgoPolOpt$ when this criterion is
satisfied.  Again, the correctness of such stopping procedure follows
from Theorem~\ref{thm:general-CI-Policy-opt} and a union bound.

\subsection{Some numerical simulations}

In the setting of $Q$-learning, the worst-case sample complexity for
any MDP to achieve $\error$-accuracy is $\frac{1}{\error^2(1 -
  \contractPar)^3}$. However Example~\ref{ExaSimple} illustrates that
even under our looser instance-dependent guarantee from
equation~\eqref{eqn:PolOptLooseUB}, we only require
$\frac{1}{\error^2(1 - \contractPar)^{3 - 2 \lambda}}$ samples. Like
the example in policy evaluation, there can still be a substantial
difference between the number of samples required for the worst-case
versus the instance-dependent case. We illustrate the gains provided
by Algorithm~\ref{EmpIRE} in the following numerical simulations.

For every combination of $(\contractPar, \diff)$, we ran
Algorithm~\ref{EmpIRE} with the ROOT-SA algorithm~\cite{MKWBJVariance22} as our
base procedure on the MDP described in Example~\ref{ExaSimple} for
$1000$ trials. We used values of $\contractPar$ that were uniformly
spaced between $0.9$ and $0.99$ on the log-scale, and the two choices
$\diff \in \{1.0, 1.5\}$. The desired tolerance was set at $\error =
0.05$. The initialization point $\qvalues_0$ was chosen via setting
aside $\frac{2}{(1 - \contractPar)^2}$ samples and estimating
$\reward$ and $\TranMatQ$ via averaging, and then solving for the
optimal $Q$-function for this MDP. We measured the factor savings by
computing the ratio of the worst-case sample complexity
$\frac{1}{\error^2(1-\contractPar)^3}$ with the number of samples used
by Algorithm~\ref{EmpIRE}. The results are presented in
Figure~\ref{plot_mdp_factorsavings}. In order to verify the
correctness of our guarantee, for each trial we also keep track of the
predicted error given by equation~\eqref{eqn:general-CI-PolOpt-1} and
the true error $\inftynorm{\Qvaluehat - \Qvaluestar}$. We see in
Figure~\ref{mdp_plot_errors} that the true error is consistently below
the predicted error and also consistently below the specified error
threshold of $0.5$. These plots illustrate the practical benefits that
Algorithm~\ref{EmpIRE} brings when taking advantage of
instance-dependent theory.

\begin{figure}[ht!]
  \begin{center}
  \widgraph{0.6\textwidth}{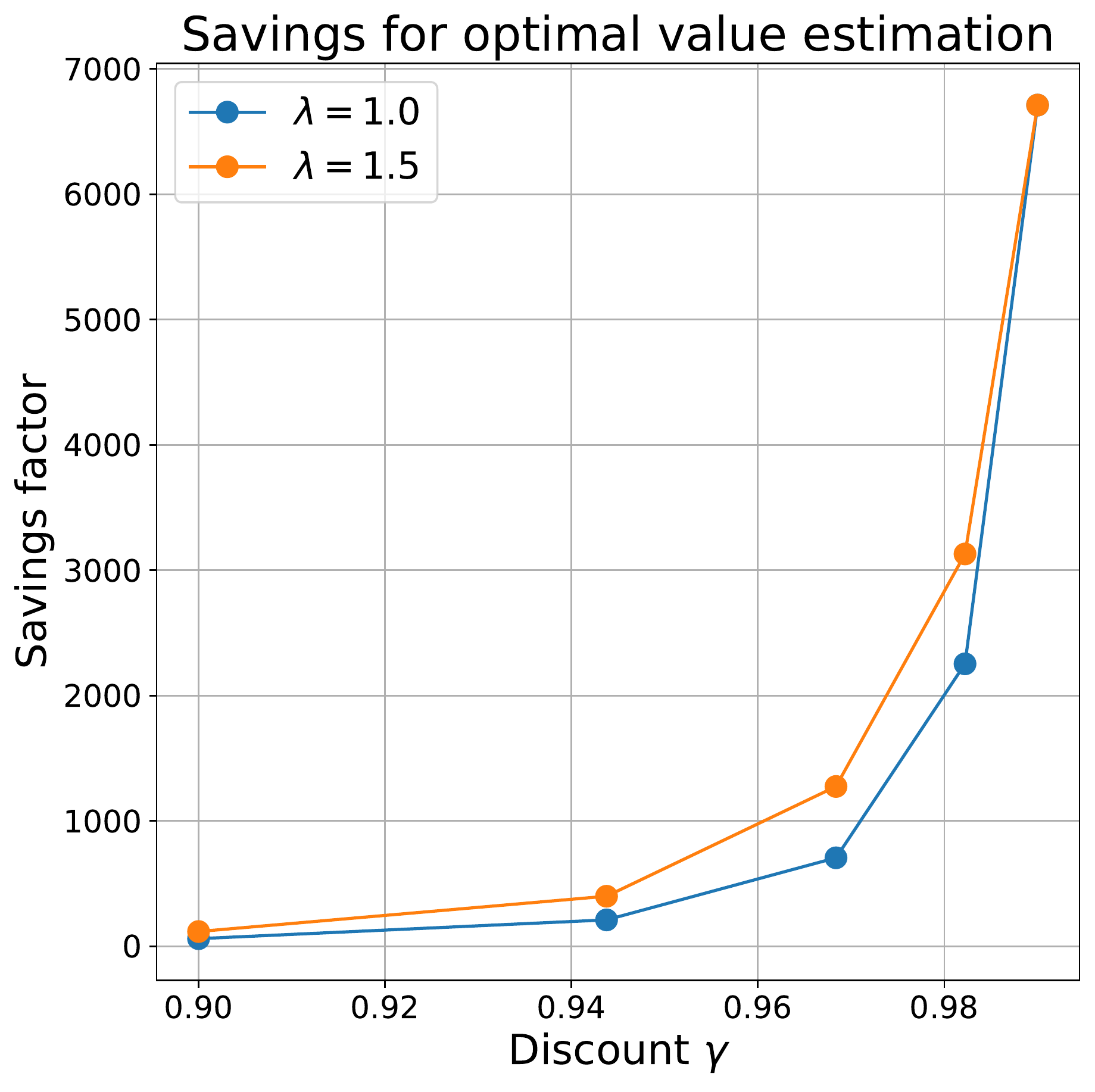}
  \caption{Illustration of the savings in sample size requirements of
    Algorithm~\ref{EmpIRE} on the MDP in Example~\ref{ExaSimple} for
    different choices of $\contractPar$ and $\diff$. The figure plots
    the factor of savings, i.e. the ratio of the number of samples
    required in the worst case to the number of samples actually used
    against the log discount complexity factor for both $\lambda =
    1.0$ (blue) and $\lambda = 1.5$ (orange).}
  \label{plot_mdp_factorsavings}
  \end{center}
\end{figure}

\begin{figure}[ht!]
  \begin{center}
  \begin{tabular}{ccc}
  \widgraph{0.5\textwidth}{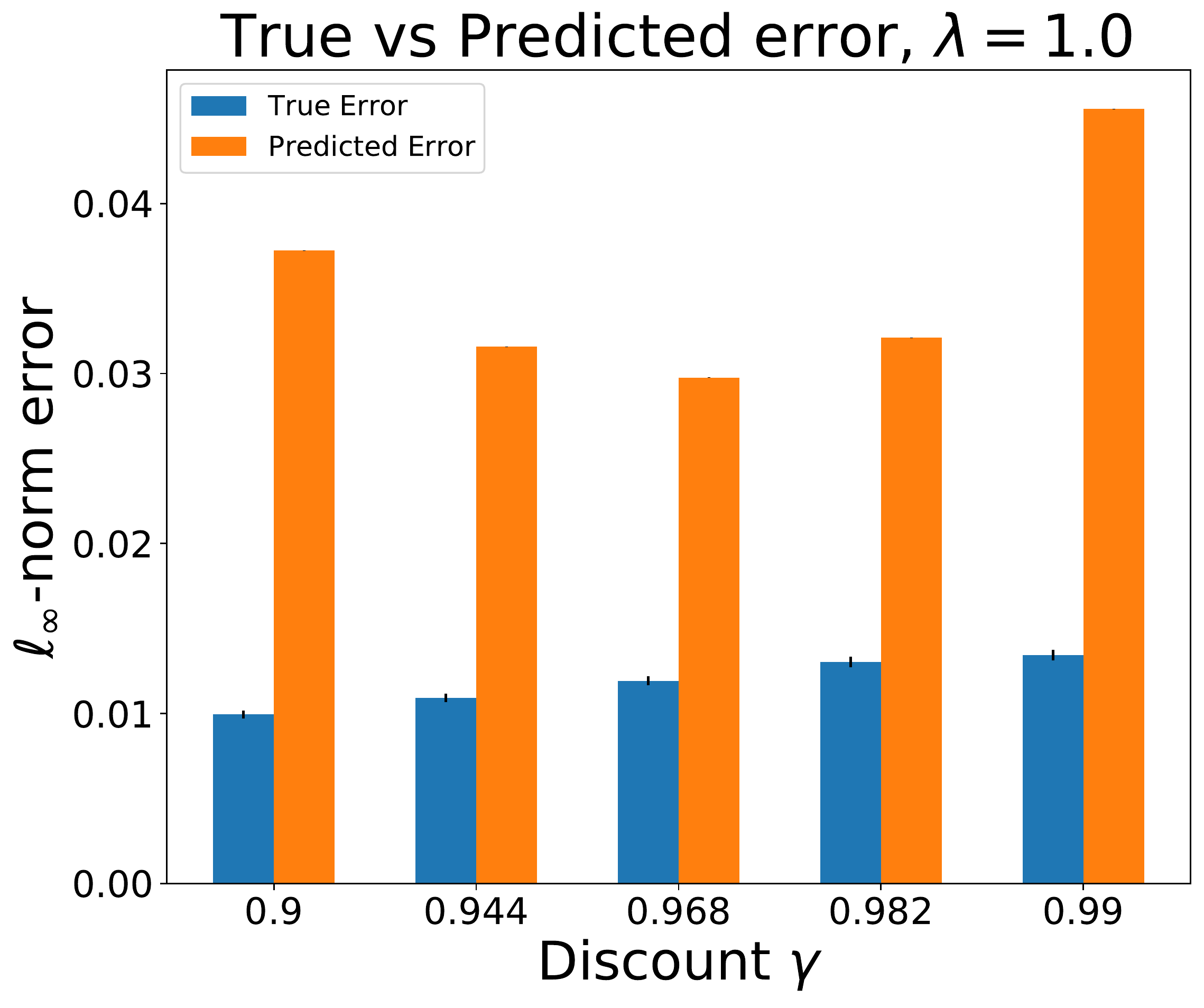}
  &&
  \widgraph{0.5\textwidth}{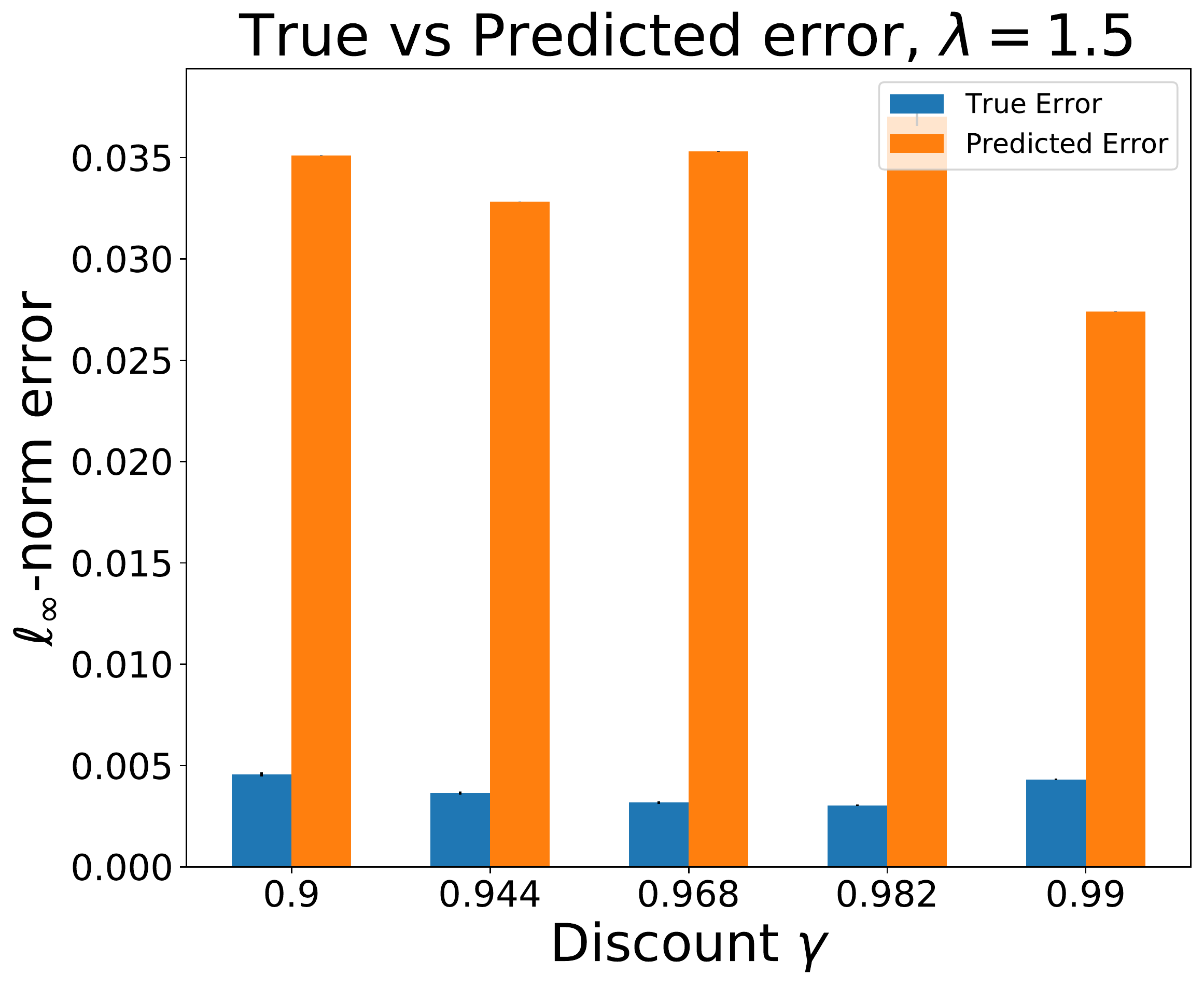}
  \\ (a) && (b)
  \end{tabular}
  \caption{Illustration of the termination behavior of
    Algorithm~\ref{EmpIRE} applied on the MDP in
    Example~\ref{ExaSimple}. Plots the average of the true error
    (blue) and predicted error (orange) along with error bars denoting
    the standard deviation for different choices of $\contractPar$ and
    for (a) $\diff = 1.0$ and (b) $\diff = 1.5$.}
\label{mdp_plot_errors}
\end{center}
\end{figure}


\section{Proofs}
\label{SecProofs}

In this section, we provide proofs of our main results, with
Sections~\ref{sec:proof-of-thm-poleval}
and~\ref{sec:Proof-of-Policy-Opt-Thm} devoted the proofs of
Theorems~\ref{thm:main-thm-poleval}
and~\ref{thm:general-CI-Policy-opt}, respectively.


\subsection{Proof of Theorem~\ref{thm:main-thm-poleval}}
\label{sec:proof-of-thm-poleval}

Let us first introduce an auxiliary result that plays a key role in
the proof:
\begin{lemma}
\label{lems:key-holdout-bound}
For a given tolerance $\pardelta \in (0,1)$, consider a pair of
hold-out estimates $\holdobsmat$ and $\secholdobsmat$, each based on
hold-out size $\holdnumobs \geq 32 \cdot \frac{\log( 4\numstates^2 /
  \pardelta)}{(1 - \contractPar)^2}$. For any positive semi-definite
matrix $\Mmat$, we have
\begin{align}
\diagnorm{(\Id - \contractPar \holdobsmat)^{-1} \Mmat (\Id -
  \contractPar \secholdobsmat)^{-\top}} \leq 3 \diagnorm{(\Id -
  \contractPar \TranMat)^{-1} \Mmat (\Id - \contractPar
  \TranMat)^{-\top}},
\end{align}
with probability at least $1 - \pardelta$.
\end{lemma}
\noindent See Section~\ref{sec:lems:key-holdout-bound-proof} for a
proof of this lemma.\\

Throughout the proof, we suppress the dependence of
$\valuehat_\numobs$ on the sample size $\numobs$. Additionally, we
adopt the shorthand notations
\begin{align*}
\valCov = \valueCovOpt, \quad \quad \covarhat(\valuestar) =
\covarhat(\valuestar; \data_{\varestsample}), \quad \text{and} \quad
\VarEst(\valuehat) = \VarEst(\valuehat, \Vdata_{\PEnumobs,
    \holdnumobs}).
\end{align*}

\vspace{15pt}

\subsubsection{Proof of Theorem~\ref{thm:main-thm-poleval} part (a)}

We begin by proving the bound~\eqref{eqn:general-CI-eqn1}. Using a
known empirical Bernstein bound (see Lemma~\ref{lem:emp-bern} in
Appendix~\ref{AppEmpBern}), we have
\begin{align}
\diagnorm{\optCovMat{\valuestar}}^{\frac{1}{2}} &= \diagnorm{(\Id -
  \contractPar \TranMat)^{-1} \valCov (\Id - \contractPar
  \TranMat)^{-\top}}^{\frac{1}{2}} \notag \\
\label{eqn:firstDecomp}
& \leq \diagnorm{(\Id - \contractPar \TranMat)^{-1}
  \covarhat(\valuestar) (\Id - \contractPar
  \TranMat)^{-\top}}^{\frac{1}{2}} + \frac{\bfun(\valuestar)}{1 -
  \contractPar} \cdot \sqrt{\frac{8 \log(\numstates /
    \pardelta)}{\PEnumobs - 1}}
\end{align}
with probability at at least $1 - \pardelta$.

Now suppose that we can prove that
\begin{multline}
  \label{eqn:LipSchitzBD}   
  \diagnorm{(\Id - \contractPar \TranMat)^{-1} \covarhat(\valuestar)
    (\Id - \contractPar \TranMat)^{-1}}^{\frac{1}{2}} \leq \sqrt{2}
  \diagnorm{(\Id - \contractPar \TranMat)^{-1} \covarhat(\valuehat)
    (\Id - \contractPar \TranMat)^{-1}}^{\frac{1}{2}} \\
  + \tfrac{\sqrt{8}}{1 - \contractPar} \inftynorm{\valuehat -
    \valuestar}
\end{multline}
Then, by combining the last bound with
Lemma~\ref{lems:key-holdout-bound}, we find that
\begin{multline*}
\diagnorm{(\Id - \contractPar \TranMat)^{-1} \covarhat(\valuestar)
  (\Id - \contractPar \TranMat)^{-\top}}^{\frac{1}{2}} \leq \sqrt{6}
\cdot \diagnorm{(\Id - \contractPar \holdobsmat)^{-1}
  \covarhat(\valuehat) (\Id - \contractPar
  \secholdobsmat)^{-\top}}^{\frac{1}{2}} \\ + \frac{\sqrt{8}}{1 -
  \contractPar} \inftynorm{\valuehat - \valuestar}
\end{multline*}
with probability $1 - \pardelta$. Substituting the last bound into
equation~\eqref{eqn:firstDecomp} and using the fact that the function
$\bfun$ is $1$-Lipschitz in the $\ell_\infty$-norm, we deduce that
\begin{multline*}
\diagnorm{\optCovMat{\valuestar}}^{\frac{1}{2}} \leq \sqrt{6} \cdot
\diagnorm{(\Id - \contractPar \holdobsmat)^{-1} \covarhat(\valuehat)
  (\Id - \contractPar \secholdobsmat)^{-\top}}^{\frac{1}{2}} +
\frac{\bfun(\valuehat)}{1 - \contractPar} \cdot \sqrt{\frac{8
    \log(\numstates / \pardelta)}{\PEnumobs - 1}} \\
+ \inftynorm{\valuehat - \valuestar} \cdot \left(\frac{\sqrt{8}}{1 -
  \contractPar} + \frac{1}{1 - \contractPar} \cdot \sqrt{\frac{8
    \log(\numstates / \pardelta)}{\PEnumobs - 1}} \right),
\end{multline*}
with probability at least $1 - 2\pardelta$.  Substituting the last
bound in the convergence rate condition~\eqref{eqn:PolEval-Algo-Cond},
invoking the sample size condition $\numobs \geq
\specfast^2(\pardelta) \cdot \frac{24 \cdot 8 \cdot \log(\numstates /
  \pardelta)}{(1 - \contractPar)^2}$ from
Theorem~\ref{thm:main-thm-poleval} yields the
claim~\eqref{eqn:general-CI-eqn1}. It remains to prove the
bound~\eqref{eqn:LipSchitzBD}.

\paragraph{Proof of the bound~\eqref{eqn:LipSchitzBD}:}

For any index $i \in [\numstates]$, we define the random variable of
interest \mbox{$U_i \defn e_i^T (\Id - \contractPar \TranMat)^{-1}
  \covarhat(\valuestar) (\Id - \contractPar \TranMat)^{-1} e_i$}.
With this definition, substituting the expression for the matrix
$\covarhat(\valuestar)$ we have
\begin{align*}
U_i & = \frac{1}{\PEnumobs(\PEnumobs - 1)} \sum_{1 \leq j < k \leq
  \PEnumobs} \left(e_i^T (\Id - \contractPar \TranMat)^{-1}
(\NoisyReward_j - \NoisyReward_k + \contractPar (\obsmat_j -
\obsmat_k) \valuestar) \right)^2 \\
& \leq \frac{2}{\PEnumobs(\PEnumobs - 1)} \sum_{1 \leq j < k \leq
  \PEnumobs} \left(e_i^T (\Id - \contractPar \TranMat)^{-1}
(\NoisyReward_j - \NoisyReward_k + \contractPar (\obsmat_j -
\obsmat_k) \valuehat) \right)^2 \\
& \qquad \qquad + \frac{2\contractPar^2}{\PEnumobs(\PEnumobs - 1)}
\sum_{1 \leq j < k \leq \PEnumobs} \left(e_i^T (\Id - \contractPar
\TranMat)^{-1} (\obsmat_j - \obsmat_k) (\valuestar - \valuehat)
\right)^2 \\
& \qquad \leq 2 e_i^T (\Id - \contractPar \TranMat)^{-1}
\covarhat(\valuehat) (\Id - \contractPar \TranMat)^{-1} e_i +
\frac{8}{(1 - \contractPar)^2} \cdot \inftynorm{\valuehat -
  \valuestar}^2.
\end{align*}
The second last inequality follows from triangle inequality, and
last inequality follows from the operator norm bound $\|(\Id -
\contractPar \TranMat)^{-1}\|_{1,\infty} \leq \frac{1}{1 -
  \contractPar}$ and the Holder's inequality. This completes the
proof of the bound~\eqref{eqn:LipSchitzBD}.

\subsubsection{Proof of Theorem~\ref{thm:main-thm-poleval} part (b)}

Here we prove the bound~\eqref{eqn:general-CI-eqn2}. Overall, the
argument is similar to that of bound~\eqref{eqn:general-CI-eqn1}. We
have
\begin{align*}
\diagnorm{\VarEst(\valuehat)}^{\frac{1}{2}} & = \diagnorm{(\Id - \contractPar
  \holdobsmat)^{-1} \covarhat(\valuehat) (\Id - \contractPar
  \secholdobsmat)^{-\top}}^{\frac{1}{2}} \\
& \stackrel{(i)}{\leq} \sqrt{3} \cdot \diagnorm{(\Id - \contractPar
  \TranMat)^{-1} \covarhat(\valuehat) (\Id - \contractPar
  \TranMat)^{-\top}}^{\frac{1}{2}} \\
& \stackrel{(ii)}{\leq} \sqrt{6} \cdot \diagnorm{(\Id - \contractPar
  \TranMat)^{-1} \covarhat(\valuestar) (\Id - \contractPar
  \TranMat)^{-\top}}^{\frac{1}{2}} + \frac{\sqrt{24}}{1 - \contractPar}
\cdot \inftynorm{\valuehat - \valuestar} \\
& \stackrel{(iii)}{\leq} \sqrt{6} \cdot \diagnorm{(\Id - \contractPar
  \TranMat)^{-1} \valCov (\Id - \contractPar
  \TranMat)^{-\top}}^{\frac{1}{2}} \\
& \qquad + \frac{\bfun(\valuestar)}{1 - \contractPar} \cdot
\sqrt{\frac{48 \log(\numstates / \pardelta)}{\PEnumobs - 1}} +
\frac{\sqrt{24}}{1 - \contractPar} \cdot \inftynorm{\valuehat -
  \valuestar}
\end{align*}
with probability at least $1 - 2\pardelta$.  Inequality (i) follows
from Lemma~\ref{lems:key-holdout-bound}, inequality (ii) follows from
the proof of bound~\eqref{eqn:LipSchitzBD}, and inequality (iii)
follows from the empirical Bernstein
lemma~\ref{lem:emp-bern}. Combining the last bound with the
bound~\eqref{eqn:PolEval-Algo-Cond} for algorithm~$\AlgoEval$ and
rearranging yields the claim~\eqref{eqn:general-CI-eqn2}.


\subsubsection{Proof of Lemma~\ref{lems:key-holdout-bound}}
\label{sec:lems:key-holdout-bound-proof}

We begin by observing that the $\diagnorm{\cdot}$ operator, despite
not being a norm, satisfies the triangle inequality. Indeed, for
square matrices $\Amat$ and $\Bmat$ with matching dimensions, we have
\begin{align}
\diagnorm{\Amat + \Bmat} &= \max_i |e_i^T (\Amat + \Bmat) e_i | \notag
\\
& \leq \max_i \left( |e_i^T \Amat e_i| + |e_i^T \Bmat e_i| \right)
\notag \\
& \leq \max_i |e_i^T \Amat e_i| + \max_i |e_i^T \Bmat e_i| \notag \\
\label{eqn:diag-triangle-ineq}
& = \diagnorm{\Amat} + \diagnorm{\Bmat}.
\end{align}
We also require the following simple lemma:

\begin{lemma}
\label{lems:holdout-infty-bound}
For holdout sample size satisfying $\holdnumobs \geq \frac{32 \log(4
  \numstates / \pardelta)}{(1 - \contractPar)^2}$ and any vector
$\values \in \real^{\numstates}$, we have
\begin{align*}
\inftynorm{(\holdobsmat - \TranMat) \values} \leq \frac{(1 -
  \contractPar)}{\sqrt{2}} \inftynorm{\values},
\end{align*}
with probability at least $1 - \frac{\pardelta}{4}$.
\end{lemma}

\begin{proof}
From Hoeffding's inequality, we have
\begin{align*}
[(\holdobsmat - \TranMat) \cdot \values]_i = \frac{1}{\holdnumobs}
\sum_{j=1}^{\holdnumobs} \langle \obsmat_j(i) - p_i, \values \rangle
\leq 4 \sqrt{\frac{\log(4 \numstates / \pardelta)}{\holdnumobs}} \cdot
\inftynorm{\values},
\end{align*}
with probability at least $1 - \frac{\pardelta}{4\numstates}$.  Here,
$\obsmat_j(i)$ and $p_i$, respectively denote the $i^{th}$ row vector
of the matrices $\obsmat_j$ and $\TranMat$. Now, applying a union
bound over $\numstates$ coordinates and using the lower bound on the
holdout sample size $\holdnumobs$ yields the claim of
Lemma~\ref{lems:holdout-infty-bound}.
\end{proof}
\vspace{10pt}

\noindent We are now ready to prove
Lemma~\ref{lems:key-holdout-bound}. For ease of notation, we use the
shorthands
\begin{align*}
\matop = \Id - \contractPar \TranMat, \quad
\holdmatop
= \Id - \contractPar \holdobsmat,
\quad \text{and} \quad
\hatmatop \defn \Id -
\contractPar \secholdobsmat.
\end{align*}
In light of the triangle inequality~\eqref{eqn:diag-triangle-ineq}, it
suffices to prove that
\begin{align*}
  \diagnorm{\matop^{-1} \Cov \matop^{-\top} - \holdmatop^{-1} \Cov
    \hatmatop^{-\top} } \leq 2 \diagnorm{\matop^{-1} \Cov \matop^{-\top}}
\end{align*}
with probability at least $1 - \pardelta$. Simple algebra and another
application of the triangle inequality~\eqref{eqn:diag-triangle-ineq}
yields
\begin{align}
\label{eqn:three-term-decomp}
\diagnorm{\matop^{-1} \Cov \matop^{-\top} - \holdmatop^{-1} \Cov
  \hatmatop^{-\top} } & \leq \diagnorm{(\matop^{-1} - \holdmatop^{-1})
  \Cov \matop^{-\top}} + \diagnorm{\holdmatop^{-1} \Cov
  (\matop^{-\top} - \hatmatop^{-\top})} \notag \\
& \leq \diagnorm{(\matop^{-1} - \holdmatop^{-1}) \Cov \matop^{-\top}} +
\diagnorm{(\matop^{-1} - \hatmatop^{-1}) \Cov \matop^{-\top}} \notag
\\
& \qquad + \diagnorm{(\matop^{-1} - \holdmatop^{-1}) \Cov
  (\matop^{-\top} - \hatmatop^{-\top})} \notag \\
& = \sum_{j=1}^3 T_j,
\end{align}
where we define
\begin{align*}
&T_1 \defn \diagnorm{(\matop^{-1} - \holdmatop^{-1}) \Cov
    \matop^{-\top}}, \qquad T_2 \defn \diagnorm{(\matop^{-1} -
    \hatmatop^{-1}) \Cov \matop^{-\top}}, \qquad \text{and} \\
& \qquad \qquad \qquad T_3 \defn \diagnorm{(\matop^{-1} -
    \holdmatop^{-1}) \Cov (\matop^{-\top} - \hatmatop^{-\top})}.
\end{align*}
We bound these three terms individually.

\paragraph{Bounding $T_1$:} We have
\begin{align*}
T_1 = \diagnorm{(\matop^{-1} - \holdmatop^{-1}) \Cov \matop^{-\top}} & =
\max_i | e_i^T \holdmatop^{-1} (\matop - \holdmatop) \matop^{-1} \Cov
\matop^{-\top} e_i | \\
& \leq \max_i \inftynorm{\holdmatop^{-1} (\matop - \holdmatop)
  \matop^{-1} \Cov \matop^{-\top} e_i} \\ &\leq \frac{1}{\sqrt{2}} \cdot
\max_i \inftynorm{\matop^{-1} \Cov \matop^{-\top} e_i} \\
& = \frac{1}{\sqrt{2}} \cdot \max_{i, j} |e_j^T \matop^{-1} \Cov
\matop^{-\top} e_i | \\
& \leq \frac{1}{\sqrt{2}} \diagnorm{\matop^{-1} \Cov \matop^{-\top}},
\end{align*}
with probability at least $1 - \frac{\pardelta}{4}$. The third line
follows from Lemma~\ref{lems:holdout-infty-bound}, a union bound over
$\Dim$ coordinates and the operator norm bound $\|\matop^{-1}\|_{1,\infty} =  \|(\Id - \contractPar
\TranMat)^{-1}\|_{1,\infty} \leq \frac{1}{1 - \contractPar}$.  The
final inequality follows from the fact that $\matop^{-1} \Cov
\matop^{-\top}$ is a covariance matrix and the maximum entry of a
covariance matrix is same as the maximum diagonal entry. Indeed, for a
zero-mean random vector $X$ with covariance matrix $\matop^{-1} \Cov
\matop^{-\top}$, we have for all indices $i, j$,
\begin{align*}
|e_j^T \matop^{-1} \Cov \matop^{-\top} e_i | = |\text{Cov}(X_j, X_i)| \leq \sqrt{\EE[X_j^2] \cdot \EE[X_i^2]} &\leq \max_k \text{Var}(X_k) = \diagnorm{\matop^{-1} \Cov \matop^{-\top}},
\end{align*}
as claimed.

\paragraph{Bounding $T_2$:} An identical calculations for the term $T_2$
in the bound~\eqref{eqn:three-term-decomp} yields
\begin{align*}
  \diagnorm{(\matop^{-1} - \hatmatop^{-1}) \Cov \matop^{-\top}}
  \leq \frac{1}{\sqrt{2}} \diagnorm{\matop^{-1} \Cov \matop^{-\top}}
\end{align*}
with probability at least $1 - \frac{\pardelta}{4}$. 

\paragraph{Bounding $T_3$:} We have
\begin{align*}
\diagnorm{(\matop^{-1} - \holdmatop^{-1})\Cov(\matop^{-\top} -
  \hatmatop^{-\top})} & = \max_i | e_i^T (\matop^{-1} -
\holdmatop^{-1})\Cov(\matop^{-\top} - \hatmatop^{-\top}) e_i | \\
& \leq \max_i \inftynorm{(\matop^{-1} -
  \holdmatop^{-1})\Cov(\matop^{-\top} - \hatmatop^{-\top}) e_i} \\
& = \max_i \inftynorm{\holdmatop^{-1}(\matop - \holdmatop) \matop^{-1}
  \Cov (\matop^{-\top} - \hatmatop^{-\top}) e_i} \\
& \stackrel{(a)}{\leq} \frac{1}{\sqrt{2}} \cdot \max_i \inftynorm{
  \matop^{-1} \Cov (\matop^{-\top} - \holdmatop^{-\top}) e_i} \\
& = \frac{1}{\sqrt{2}} \cdot \max_{i, j} |e_j^T \matop^{-1}
\Cov(\matop^{-\top} - \hatmatop^{-\top}) e_i | \\
& = \frac{1}{\sqrt{2}} \cdot \max_j \inftynorm{(\matop^{-1} -
  \hatmatop^{-\top}) \Cov \matop^{-\top} e_j} \\
& \stackrel{(b)}{\leq} \frac{1}{2} \cdot \max_j \inftynorm{\matop^{-1}
  \Cov \matop^{-\top} e_j} \\
& \leq \frac{1}{2} \diagnorm{\matop^{-1} \Cov
  \matop^{-\top}},
\end{align*}
Inequality (a) follows from Lemma~\ref{lems:holdout-infty-bound}
conditioned on the randomness on in the matrix $\hatmatop$; recall
that matrices $\hatmatop$ and $\holdmatop$ is independent by
construction.  Inequality (b) follows again by applying the
Lemma~\ref{lems:holdout-infty-bound}. Invoking a union bound, we have
that the last bound holds with probability at least $1 -
\frac{\pardelta}{2}$.  Putting together the pieces yields
\begin{align*}
\diagnorm{\holdmatop^{-1} \Cov \hatmatop^{-\top}} \leq \left(1 + \sqrt{2}
+ \frac{1}{2} \right) \cdot \diagnorm{\matop^{-1} \Cov \matop^{-\top}}
\leq 3 \cdot \diagnorm{\matop^{-1} \Cov \matop^{-\top}}
\end{align*}
with probability at least $1 - \pardelta$, as claimed.

\subsection{Proof of Theorem~\ref{thm:general-CI-Policy-opt}}
\label{sec:Proof-of-Policy-Opt-Thm}

The proof of this theorem is similar to the proof of
Theorem~\ref{thm:main-thm-poleval}, and it is based on a Lipschitz
property of the empirical covariance matrix operator
$\diagnorm{\QVarEst(\cdot)}^{\frac{1}{2}}$ (see
Lemma~\ref{QCov-Lipschitz-lemma}) and an empirical Bernstein lemma
(see Lemma~\ref{lem:emp-bern}).  For notational simplicity, we use
$\Qvaluehat$ as a shorthand for $\Qvaluehat_\numobs$, and similarly,
$\QVarEst(\Qvaluehat)$ in place of $\QVarEst(\Qvaluehat_\numobs,
\Qdata_\numobs)$. \\

It suffices to prove that the following two bounds hold, each with
probability at least $1 - \pardelta$:
\begin{subequations}
\begin{align}
  \label{eqn:PolOpt-main-1}
   \diagnorm{\optCovMatQ{\Qvaluestar}}^{\frac{1}{2}} & \leq
   \inftynorm{\QVarEst(\Qvaluehat)}^{\frac{1}{2}} +
   \bfun(\Qvaluehat) \cdot \sqrt{\frac{8 \log(\Dim /
       \pardelta)}{\PEnumobs - 1}} + \inftynorm{\Qvaluehat -
     \Qvaluestar} \cdot \left( \sqrt{\tfrac{8 \log(\Dim /
       \pardelta)}{\PEnumobs - 1}} + \sqrt{8} \right) \\
\inftynorm{\QVarEst(\Qvaluehat)}^{\frac{1}{2}} &\leq \sqrt{2} \cdot
\diagnorm{\optCovMatQ{\Qvaluestar}}^{\frac{1}{2}} +
\bfun(\Qvaluestar) \cdot \sqrt{\tfrac{16 \log(\Dim /
    \pardelta)}{\PEnumobs - 1}} + \sqrt{8} \inftynorm{\Qvaluehat -
  \Qvaluestar},
\label{eqn:PolOpt-main-2}
\end{align}
\end{subequations}
Indeed, combining the bound~\eqref{eqn:PolOpt-main-1} with the
condition~\eqref{eqn:PolOpt-algo-prop} of the algorithm~$\AlgoPolOpt$
and the sample size lower bound $\numobs \geq c_1(\AlgoEval,
\pardelta)^2 \cdot \frac{32 \cdot \log(\Dim / \pardelta)}{(1 -
  \contractPar)^2}$, we have
\begin{align*}
  \inftynorm{\Qvaluehat - \Qvaluestar} &\leq \frac{2 \sqrt{2} \cdot
    \specfast(\pardelta)}{\sqrt{\PEnumobs}} \cdot
  \frac{\diagnorm{\QVarEst(\Qvaluehat)}^{\frac{1}{2}}}{1 -
    \contractPar} + \frac{2 c_1(\AlgoEval) \bfun(\Qvaluehat)}{1 -
    \contractPar} \cdot \frac{ \sqrt{ 16 \log(\Dim /
      \pardelta)}}{\PEnumobs - 1} \notag + \frac{2
    \specslow(\pardelta)}{\numobs}.
\end{align*}
with probability at least $1 - 2\pardelta$.  Furthermore, since
function $\bfun$ is $\contractPar$-Lipschitz in the
$\ell_\infty$-norm, we have
\begin{align*}
  |\bfun(\Qvalues) - \bfun(\Qvaluestar)| \leq \contractPar \cdot
  \inftynorm{\Qvalues - \Qvaluestar}.
\end{align*}
Combining this inequality with the bound~\eqref{eqn:PolOpt-main-2} and
the condition~\eqref{eqn:PolOpt-algo-prop} yields the claimed
bound~\eqref{eqn:general-CI-PolOpt-2}. \\

It remains to prove the bounds~\eqref{eqn:PolOpt-main-1}
and~\eqref{eqn:PolOpt-main-2}.  In doing so, we make use of the
following auxiliary lemma:
\begin{lemma}
\label{QCov-Lipschitz-lemma}
For any pair $\Qvalues_1, \Qvalues_2$, we have
\begin{align*}
\diagnorm{\QVarEst(\Qvalues_1)}^{\frac{1}{2}} & \leq \sqrt{2}
\diagnorm{\QVarEst(\Qvalues_2)}^{\frac{1}{2}} + \sqrt{8}
\inftynorm{\Qvalues_1 - \Qvalues_2}.
\end{align*}
\end{lemma}
\noindent See Section~\ref{SecProofLemQcovLip} for the proof of this
claim.

\subsubsection{Proof of the bound~\eqref{eqn:PolOpt-main-1}}

We have
\begin{multline*}
\diagnorm{\optCovMatQ{\Qvaluestar}}^{\frac{1}{2}} \stackrel{(i)}{\leq}
\inftynorm{\QVarEst(\Qvaluestar)}^{\frac{1}{2}} +
\tfrac{\bfun(\Qvaluestar)}{1 - \contractPar} \cdot \sqrt{\tfrac{8
    \log(\Dim / \pardelta)}{\numobs - 1}} \\
 \stackrel{(ii)}{\leq} \sqrt{2}
 \diagnorm{\QVarEst(\Qvaluehat)}^{\frac{1}{2}} + \sqrt{8} \cdot
 \inftynorm{\Qvaluehat - \Qvaluestar} + \bfun(\Qvaluehat) \cdot
 \sqrt{\tfrac{8 \log(\Dim / \pardelta)}{\numobs - 1}} \\
 +
 \inftynorm{\Qvaluehat - \Qvaluestar} \cdot \sqrt{\tfrac{8\log(\Dim /
     \pardelta)}{\numobs - 1}},
\end{multline*}
with probability at least $1 - \pardelta$. Inequality (i) follows by
applying an empirical Bernstein bound (cf. Lemma~\ref{lem:emp-bern} in
Appendix~\ref{AppEmpBern}) to each diagonal entry of the matrix
$\QVarEst(\valuestar)$, combined with a union bound on all $\Dim$
diagonal entries.  (See the proof of
Theorem~\ref{thm:main-thm-poleval} for an analogous calculation.)
Inequality (ii) follows Lemma~\ref{QCov-Lipschitz-lemma} on Lipschitz
properties of the empirical variance estimate.

\subsubsection{Proof of the bound~\eqref{eqn:PolOpt-main-2}}

The proof of this claim is similar to that of the
bound~\eqref{eqn:PolOpt-main-1}. We have
\begin{align*}
\diagnorm{\QVarEst(\Qvaluehat)}^{\frac{1}{2}} & \stackrel{(i)}{\leq}
\sqrt{2} \cdot \diagnorm{\QVarEst(\Qvaluestar)}^{\frac{1}{2}} +
\sqrt{8} \inftynorm{\Qvaluehat - \Qvaluestar} \\
& \stackrel{(ii)}{\leq} \sqrt{2} \cdot
\diagnorm{\optCovMatQ{\Qvaluestar}}^{\frac{1}{2}} +
\bfun(\Qvaluestar) \cdot \sqrt{\frac{16 \log(\Dim /
    \pardelta)}{\PEnumobs - 1}} + \sqrt{8} \cdot \inftynorm{\Qvaluehat
  - \Qvaluestar}
\end{align*}
with probability at least $1 - \pardelta$.  Here the inequality (i)
follows from the Lemma~\ref{QCov-Lipschitz-lemma}, and inequality (ii)
follows from an empirical Bernstein bound (see
Lemma~\ref{lem:emp-bern} in Appendix~\ref{AppEmpBern}) combined with
the union bound.

\subsubsection{Proof of Lemma~\ref{QCov-Lipschitz-lemma}}
\label{SecProofLemQcovLip}

Given a square matrix $\Mmat$, the quantity $\| \Mmat \|_{\diag}$ only
depends on the diagonal elements of the matrix $\Mmat$.  For any
state-action pair $\stateActionPair = (\state, \action)$, let
$\QVarEst(\Qvalues)(\stateActionPair)$ denote the diagonal entry of
the covariance matrix $\QVarEst(\Qvalues)$ associated with the state
action pair $\stateActionPair$. Substituting the definition of
$\QVarEst(\Qvalues)$ from equation~\eqref{eqn:QVarEst-defn} yields
\begin{align*}
\QVarEst(\Qvalues_1)(\stateActionPair)& = \tfrac{1}{\numobs(\numobs -
  1)} \sum_{1 \leq i < \ell \leq \numobs}
\left(\NoisyOptOp_i(\Qvalues_1)(\stateActionPair) -
\NoisyOptOp_\ell(\Qvalues_1)(\stateActionPair) \right)^2 \\
&= \tfrac{1}{\numobs(\numobs -
  1)} \sum_{1 \leq i < \ell \leq \numobs}
\left(\NoisyOptOp_i(\Qvalues_2)(\stateActionPair) -
\NoisyOptOp_\ell(\Qvalues_2)(\stateActionPair) + \NoisyOptOp_i(\Qvalues_1)(\stateActionPair) - \NoisyOptOp_i(\Qvalues_2)(\stateActionPair) + \NoisyOptOp_\ell(\Qvalues_2)(\stateActionPair) - \NoisyOptOp_\ell(\Qvalues_1)(\stateActionPair) \right)^2 \\
& \stackrel{(i)}{\leq} 2 \QVarEst(\Qvalues_2)(\stateActionPair) +
\tfrac{4}{\numobs(\numobs -1 )} \sum_{1 \leq i < \ell \leq \numobs}
\left\{ ((\NoisyOptOp_i(\Qvalues_2)-
\NoisyOptOp_i(\Qvalues_1))(\stateActionPair))^2 +
((\NoisyOptOp_\ell(\Qvalues_2) -
\NoisyOptOp_\ell(\Qvalues_1)(\stateActionPair))^2 \right\} \\
& \stackrel{(ii)}{\leq} 2 \QVarEst(\Qvalues_2)(\stateActionPair) + 8
\contractPar^2 \inftynorm{\Qvalues_1 - \Qvalues_2}^2,
\end{align*}
where step (i) follows from the elementary inequality $(a + b)^2 \leq
2a^2 + 2b^2$; and step (ii) uses the fact that the noisy Bellman
operator $\NoisyOptOp_\ell$ is $\contractPar$-Lipschitz in the
$\ell_\infty$-norm. This completes the proof.


\section{Discussion}
\label{SecDiscussion}

Our work addresses the problem of obtaining instance-dependent
confidence regions for the policy evaluation problem and the optimal
value estimation problem of an MDP, given access to an instance
optimal algorithm.  The confidence regions are constructed by
estimating the instance-dependent functionals that control problem
difficulty in a local neighborhood of the given problem instance. For
both problems, the instance-dependent confidence regions are shown to
be significantly shorter for problems with favorable structure.

Our results also leave a few interesting questions. For instance, our
results and proof techniques heavily rely on the tabular nature of the
MDP.  It will be interesting to see if such data-dependent guarantees
can be extended to policy evaluation and optimal value estimation
problems when a function approximation step is involved.  Another
interesting direction could be improving the bound from
Theorem~\ref{thm:general-CI-Policy-opt} for the optimal value
estimation problem.  Additionally in the setting of policy evaluation,
we believe that the need for two independent holdout sets is
unnecessary and more likely a proof defect; we conjecture it suffices
to use one holdout set to estimate $(\Id - \contractPar
\TranMat)^{-1}$ on both the left and right side.


\bibliography{bibliography} \bibliographystyle{amsalpha}



\appendix

\section{Auxiliary Lemmas}
\label{sec:aux-lemmas}

In this section, we state the auxiliary lemmas that are used in the
main part of the paper.

\subsection{Empirical Bernstein}
\label{AppEmpBern}

The following empirical Bernstein bound is a re-statement of Theorem
10 from the paper~\cite{maurerpontil2009}:
\begin{lemma}
\label{lem:emp-bern}
Let $\{Z_i\}_{i=1}^\numobs$ be an i.i.d. sequence of real valued
random variables taking values in the unit interval $[0,1]$, and
define the variance estimate $\VarEst(Z) = \frac{1}{\numobs
  (\numobs-1)} \sum_{1 \leq i < j \leq \numobs} (Z_i - Z_j)^2$.  Then
for any $\pardelta \in (0,1)$, we have
\begin{align}
\Big| \sqrt{\EE[\VarEst(Z)]} - \sqrt{\VarEst(Z)} \Big| & \leq
\sqrt{\frac{2\log(1 / \pardelta)}{n - 1}}
\end{align}
with probability at least $1 - 2 \pardelta$.
\end{lemma}


\subsection{Calculations for Example~\ref{ExaSimple}}
\label{sec:ExaSimpleCalulations}
Here we derive the bound~\eqref{eqn:simple-instance-bound}.  Letting
$\Qvaluestar$ denote the value function of the optimal policy
$\policystar$, we have
\begin{align}
 (\obsmatQ^{\policystar} - \TranMatQ^{\policystar}) \Qvaluestar
  = \begin{bmatrix} | & | \\ (\obsmatQ_{\action_1} -
    \TranMatQ_{\action_1}) \Vstar & 0 \\ | & |\end{bmatrix}.
\end{align}
Letting $\Wmat = (\Id - \contractPar \TranMatQ_{\action_1})^{-1}
(\obsmatQ_{\action_1} - \TranMatQ_{\action_1}) \Qvalues_{\policystar}$
and solving for $(\Id - \contractPar \TranMatQ^{\policystar}) \randmat
= \contractPar (\obsmatQ^{\policystar} - \TranMatQ^{\policystar})
\Qvaluestar$ gives
\begin{align}
    \randmat = \contractPar \cdot \begin{bmatrix} | & | \\ \Wmat &
      \contractPar \Wmat \\ | & |
    \end{bmatrix}.
\end{align}
Finally, a simple calculation yields
\begin{align*}
  \var(\NoisyOptOp(\Qvaluestar)(x_1,u_1)) = p(1 - p) \cdot \frac{(1 -
    \taupar)^2}{(1 - \contractPar p)^2} , \qquad
  &\var(\NoisyOptOp(\Qvaluestar)(x_2,u_1)) =
  0,\\ \var(\NoisyOptOp(\Qvaluestar)(x_1,u_2)) = 0, \quad \text{and}
  \quad &\var(\NoisyOptOp(\Qvaluestar)(x_2,u_2)) = 0.
\end{align*}
Substituting $\taupar = 1 - (1 - \contractPar)^\lampar$ yields the
claimed bound.


\section{Proof of Corollary~\ref{cors:poleval-stopping-guarantee}}
\label{sec:CI-cor-proof}

The claim of corollary~\ref{cors:poleval-stopping-guarantee} follows from Corollary~\ref{cors:poleval-stopping-guarantee2}, the 
condition~\eqref{eqn:fast-and-slow-factors-relation} and the fact that the functions $\specfast$
and $\specslow$ are lower bounded by 1.

\begin{cors}
\label{cors:poleval-stopping-guarantee2}
Given any algorithm~$\AlgoEval$ satisfying
condition~\eqref{eqn:PolEval-Algo-Cond}, target accuracy $\error$, and
tolerance probability $\deltaTol$. Let $\valuehat$ denote the output
of algorithm~\ref{EmpIRE} with input pair $(\error, \deltaTol)$. Then
the following statements hold:
\begin{enumerate}
  \item[(a)] The estimate $\valuehat$ is $\error$-accurate in the
    $\ell_\infty$-norm:
\begin{subequations}
\begin{align}
\label{eqn:corollary-CI-eqn1}
\inftynorm{\valuehat - \valuestar} \leq \error,
\end{align}
with probability exceeding $1 - \deltaTol$.

  \item[(b)] Define the non-negative integer sets
  \begin{align}
  \label{eqn:corollary-CI-eqn2}
  A &= \left\{\epoch : \epoch + \log_2(\log(4\numstates/\pardelta_\epoch)) \geq  \log_2 \left( \frac{(1 - \contractPar)^2}{\error^2} \cdot \diagnorm{\optCovMat{\valuestar}} \right) + 4 \right\}, \qquad \text{and} \notag \\
  B &= \left\{ \epoch:  \epoch + \log_2(\log(4\numstates/\pardelta_\epoch)) \geq \log_2 \left(\frac{(1 - \contractPar)^2}{\error
  \specfast^2(\pardelta_\epoch) } \cdot \left[\frac{\specslow(\pardelta_\epoch)}{4} +
   \frac{64 \bfun(\valuestar)}{1 - \contractPar} \cdot \sqrt{\log(8
    \numstates / \pardelta_\epoch)} \right] \right) \right\}.
  \end{align}
  The algorithm~\ref{EmpIRE} terminates in at most
  \begin{align}
  \label{eqn:epoch-infimum-defn}
    \epochs = \inf \left \{ A \cap B \right \}
  \end{align}
  epochs with probability exceeding $1 - \deltaTol$.

  \item[(c)] For universal constants $(c_1, c_2, c_3)$, the
    algorithm~\ref{EmpIRE} requires at most
\begin{align}
\label{eqn:corollary-CI-eqn3}
\PEnumobs \leq
\max\left\{\frac{c_1\specfast^2(\pardelta_\epochs)}{\error^2} \cdot
\diagnorm{\optCovMat{\valuestar}}, \frac{1}{\error} \left[ c_2
  \specslow(2\pardelta_\epochs) + c_3 \frac{\bfun(\valuestar)}{1 -
    \contractPar} \sqrt{\log(4 \numstates / \pardelta_\epochs)}
  \right] \right\}
\end{align}
samples with probability exceeding $1 - \deltaTol$.

\end{subequations}
\end{enumerate}
\end{cors}

\paragraph{Remarks:} Note that in instance-optimal algorithms
$\specfast(\pardelta)$ and $\specfast(\pardelta)$ are typically
logarithmic functions of $1 / \pardelta$ such as variance-reduced
policy evaluation~\cite{khamaru2020PE} or
ROOT-SA~\cite{MKWBJVariance22}, ensuring that $\epochs$ exists and our
algorithm terminates. In both equations~\eqref{eqn:corollary-CI-eqn1}
and~\eqref{eqn:corollary-CI-eqn2}, typically the first term involving
$\frac{\diagnorm{\optCovMat{\valuestar}}}{\error^2}$ is the dominant
term and establishes a sample complexity of $\bigOh\left(
\frac{\diagnorm{\optCovMat{\valuestar}}}{\error^2} \right)$.

\paragraph{Proof of Corollary~\ref{cors:poleval-stopping-guarantee2}:}

Taking for granted now that the algorithm terminates, observe that
equation~\eqref{eqn:corollary-CI-eqn1} follows from algorithm's
termination in $\epochs$ epochs, applying a union bound over
equation~\eqref{eqn:general-CI-eqn1} from
Theorem~\ref{thm:main-thm-poleval} for epochs $\epoch = 1, \ldots,
\epochs$ and equation~\eqref{eqn:general-CI-eqn2} for the final epoch
$\epochs$ yields the claim with probability exceeding $1 - \deltaTol$.

We now turn to the claim~\eqref{eqn:corollary-CI-eqn2}.  Note that the
algorithm terminates at epoch $\epoch$ only if
\begin{align*}
\vfast + \vslow = \frac{2\sqrt{6} \cdot \specfast(\pardelta_\epoch)
}{\sqrt{\PEbatch_\epoch}} \cdot \VarEst(\valuehat, \MyData{}) +
\frac{2 \specslow(\pardelta_\epoch)}{\PEbatch_\epoch} + \frac{6
  \bfun(\valuehat)}{1 - \contractPar} \cdot \frac{\sqrt{\log(8
    \numstates / \pardelta_\epoch)}}{\PEbatch_\epoch - 1} \leq \error.
\end{align*}
By equation~\eqref{eqn:general-CI-eqn2}, we have
\begin{align*}
\vfast + \vslow \leq \frac{13 \specfast(\pardelta_\epoch)}{\sqrt{\PEbatch_\epoch}} \cdot \diagnorm{\optCovMat{\valuestar}}^{\frac{1}{2}} + \frac{2 \specslow(\pardelta_\epoch)}{\PEbatch_\epoch} + \frac{33 \bfun(\valuestar)}{1 - \contractPar} \cdot \frac{\sqrt{\log(8 \numstates / \pardelta_\epoch)}}{\PEbatch_\epoch}.
\end{align*}
Consequently, the stopping condition $\vfast + \vslow \leq \error$
holds as long as
\begin{subequations}
\begin{align}
\label{eqn:cors-proof-sample-size}
\frac{13 \specfast(\pardelta_\epoch)}{\sqrt{\PEbatch_\epoch}} \cdot
\diagnorm{\optCovMat{\valuestar}}^{\frac{1}{2}} & \leq
\frac{\error}{2}, \qquad \text{and} \\ \frac{2
  \specslow(\pardelta_\epoch)}{\PEbatch_\epoch} + \frac{33
  \bfun(\valuestar)}{1 - \contractPar} \cdot \frac{\sqrt{\log(8
    \numstates / \pardelta_\epoch)}}{\PEbatch_\epoch} & \leq
\frac{\error}{2}.
\end{align}
\end{subequations}
Since $\PEbatch_\epoch = 2^\epoch \specfast^2(\pardelta_\epoch) \cdot
\frac{32 \log(4 \numstates / \pardelta_\epoch)}{(1 - \contractPar)^2}$
by definition, equation~\eqref{eqn:corollary-CI-eqn2} follows from
plugging in $\PEbatch_\epoch$ and solving for when the above
expression is less than $\error$. The termination condition comes from
$\epochs$ being the first $\epoch$ such that both conditions above
hold.

When the algorithm terminates in $\epochs$ epochs, then the total
number of samples used can be bounded as
\begin{align}
\label{eqn:total-samples-used}
\sum_{\epoch = 1}^{\epochs} (\PEbatch_\epoch + 2 \doublehold{\epoch})
&\leq c_1 \cdot \frac{2^{\epochs}}{(1 - \contractPar)^2} \cdot
\specfast^2(\pardelta_\epochs) \cdot \log(4 \numstates /
\pardelta_\epochs) \; = \; c_1 \cdot \frac{2^{\epochs + \log_2(\log(4
    \numstates / \pardelta_\epochs))}}{(1 - \contractPar)^2} \cdot
\specfast^2(\pardelta_\epochs).
\end{align}
for some universal constant $c_1$.
Here, we have used the assumption that 
the maps $\delta \mapsto \specfast(\delta)$ and 
$\delta \mapsto \specslow(\pardelta_\epoch)$ are increasing functions of $\delta$; 
consequently for all $m = 1, 2, \ldots, \epochs$, 
\begin{align*}
  \specfast(\pardelta_\epoch)
  \leq \specfast(\pardelta_\epochs)
  \qquad \text{and} \qquad
  \specslow(\pardelta_\epoch)
  \leq \specslow(\pardelta_\epochs). 
\end{align*}
Recall that of $\epochs$ is the infimum of the two sets $A$ and $B$,
defined in~\eqref{eqn:corollary-CI-eqn2}, 
thus any integer smaller than $\epochs$,
in particular $M - 1$ satisfies the following upper bound 
\begin{align*}
\epochs-1 + \log_2(\log(4 \numstates / \pardelta_{\epochs - 1})) &\leq c_2 + \max\left\{ \log_2 \left( \frac{(1 -
  \contractPar)^2 }{\error^2} \cdot \diagnorm{\optCovMat{\valuestar}} \right)
 , \right . \notag \\ & \left. \qquad \qquad 
\log_2 \left(\frac{(1 - \contractPar)^2}{\error
  \specfast^2(\pardelta_{\epochs - 1})} \cdot \left[c_3 \specslow(\pardelta_{\epochs - 1}) +
  c_4 \frac{ \bfun(\valuestar)}{1 - \contractPar} \cdot \sqrt{\log(8
    \numstates / \pardelta_{\epochs - 1})} \right] \right) 
\right\}
\end{align*}
for some universal constants $(c_2, c_3, c_4)$. Rewriting the last bound in terms 
of $\epochs$ and using the relation $\pardelta_{\epochs - 1} = 2\pardelta_{\epochs}$ we obtain
\begin{align*}
\epochs + \log_2(\log(2 \numstates / \pardelta_{\epochs})) &\leq c_2 + \max\left\{ \log_2 \left( \frac{(1 -
  \contractPar)^2 }{\error^2} \cdot \diagnorm{\optCovMat{\valuestar}} \right)
 , \right . \notag \\ & \left. \qquad \qquad 
\log_2 \left(\frac{(1 - \contractPar)^2}{\error
  \specfast^2(2\pardelta_{\epochs})} \cdot \left[c_3 \specslow(2\pardelta_{\epochs}) +
  c_4 \frac{ \bfun(\valuestar)}{1 - \contractPar} \cdot \sqrt{\log(4
    \numstates / \pardelta_{\epochs})} \right] \right) 
\right\}
\end{align*}
Substituting the last bound on 
$\epochs$ into equation~\eqref{eqn:total-samples-used} and using
the fact that $\delta \mapsto \specfast(\delta)$ is a decreasing function of
$\pardelta$  yields equation~\eqref{eqn:corollary-CI-eqn3}. This 
completes the proof of Corollary~\ref{cors:poleval-stopping-guarantee2}. 

\section{Policy Optimization: Further Details}
\label{sec:VRQ-prop-proof}
In this section, we state and prove a bound for the variance
reduced~$Q$-learning algorithm studied in
papers~\cite{wainwright2019variancereduced,khamaru2021instance}.
The goal is to to show that the variance reduced $Q$-learning
algorithm satisfies the condition in
equation~\eqref{eqn:PolOpt-algo-prop}.  Much of the content of this
section is directly borrowed from the
paper~\cite{khamaru2021instance}.  Throughout this section,
we use the shorthand $\Dim \defn \numstates \cdot |\actionset|$.

\subsection{Variance-reduced $Q$-learning}
We start by restating the variance reduced $Q$-learning algorithm 
from the paper~\cite{khamaru2021instance}. 
See the papers~\cite{khamaru2021instance,wainwright2019variancereduced} for a motivation of the algorithm. 

\paragraph{A single epoch:}
The epochs are indexed with
integers \mbox{$m = 1, 2, \ldots, \epochs$}, where $\epochs$
corresponds to the total number of epochs to be run. Each epoch $m$
requires the following four inputs:
\begin{itemize}
\item an element $\Qbar$, which is chosen to be the output of the
  previous epoch $m - 1$;
\item a positive integer $\epochlength$ denoting the number of steps
  within the given epoch;
\item a positive integer $\recentersize$ denoting the batch size used
  to calculate the Monte Carlo update:
  \begin{align}
    \label{EqnMCRecenter}
\RecenterOp(\Qbar_m) \defn \frac{1}{\recentersize} \sum_{i \in
  \recentersample} \NoisyOptOp_i(\Qbar_m).
\end{align}
\item a set of fresh operators $\{\NoisyOptOp_i\}_{i \in
  \epochsampleset}$, with \mbox{$|\epochsampleset| = \recentersize +
  \epochlength$}. The set $\epochsampleset$ is partitioned into two
  subsets having sizes $\recentersize$ and $\epochlength$,
  respectively. The first subset, of size $\recentersize$, which we
  call $\recentersample$, is used to construct the Monte Carlo
  approximation~\eqref{EqnMCRecenter}. The second subset, of size
  $\epochlength$ is used to run the $\epochlength$ steps within the
  epoch.
\end{itemize}

We summarize a single epoch in pseudocode form in
Algorithm~\ref{AlgRunEpoch}.
\begin{varalgorithm}{SingleEpoch}
\caption{\qquad RunEpoch $(\Qbar; \epochlength, \recentersize,
  \{\NoisyOptOp_i\}_{i\in\epochsampleset})$}
\label{AlgRunEpoch}
\begin{algorithmic}[1]
\STATE Given (a) Epoch length $\epochlength$, (b) Re-centering vector
$\Qbar$, (c) Re-centering batch size $\recentersize$, \mbox{(d)
  Operators} $\{\NoisyOptOp_i\}_{i \in \epochsampleset}$ \STATE Compute
the re-centering quantity
\begin{align*}
\RecenterOp(\Qbar) \defn \frac{1}{\recentersize} \sum \limits_{i \in
  \recentersample} \NoisyOptOp_i(\Qbar)
\end{align*}
\STATE Initialize $\Qvalues_1 =  \Qbar$
\FOR{$k = 1, 2, \ldots, \epochlength$}
\STATE Compute the variance-reduced update:
\begin{align*}
  \Qvalues_{k+1} &=  (1 - \stepsize_k) \Qvalues_k + \stepsize_k
\left\{\NoisyOptOp_k(\Qvalues_k) - \NoisyOptOp_k(\Qbar_m) +
\RecenterOp(\Qbar_m)\right\}, \quad \text{with stepsize } \, \stepsize_k = \frac{1}{1 + (1 - \contractPar)k}.
\end{align*}
\ENDFOR
\STATE \textbf{return} $\ele_{\epochlength + 1}$
\end{algorithmic}
\end{varalgorithm}

\paragraph{Overall algorithm:}
The overall algorithm, denoted by~\ref{AlgVR} for short, has five
inputs: (a) an initialization $\Qbar_1$, (b) an integer $\NumEpoch$,
denoting the number of epochs to be run, (c) an integer
$\epochlength$, denoting the length of each epoch, (d) a sequence of
batch sizes $\{\recentersize\}_{m=1}^\epochlength$, denoting the
number of operators used for re-centering in the $\NumEpoch$ epochs,
and (e) sample batches $\{\{\NoisyOptOp_i\}_{i \in
  \epochsampleset}\}_{m=1}^\NumEpoch$ to be used in the $\NumEpoch$
epochs.  Given these five inputs, the overall procedure can be
summarized as in Algorithm~\ref{AlgVR}.

\begin{varalgorithm}{\VRQL}
\caption{}
\label{AlgVR}
\begin{algorithmic}[1]
\STATE Given (a) Initialization $\Qbar_1$, (b) Number of epochs,
$\epochs$,  (c) Epoch length $\epochlength$,  (d)
Re-centering sample sizes $\{\recentersize\}_{m=1}^M$, (e) Sample
batches $\{\NoisyOptOp_i\}_{i \in \epochsampleset}$ for \mbox{$m = 1,
  \ldots, \epochs$}
\STATE Initialize at $\Qbar_1$ \FOR{$m = 1, 2, \ldots, \epochs$}
\STATE $\Qbar_{m+1} = \text{RunEpoch}(\Qbar_m; \epochlength,
\recentersize, \{\NoisyOptOp_i\}_{i \in \epochsampleset})$ \ENDFOR \STATE
\textbf{return} $\Qbar_{M+1}$ as final estimate
\end{algorithmic}
\end{varalgorithm}


\paragraph{Input parameters} 
Given a tolerance probability $\pardelta \in (0, 1)$ and the number of
available i.i.d.\ samples $\Qnumobs$, we run Algorithm~\ref{AlgVR}
with a total of $\epochs \defn \log
\left(\frac{\Qnumobs(1-\contractPar)^2}{8\log((16\Dims/\pardelta)
  \cdot \log \Qnumobs)}\right)$ epochs, along with the following
parameter choices: \\
\begin{subequations}\label{EqnQParam}
\texttt{Re-centering sizes:}
\begin{align}\label{EqnPERecenterSize}
\begin{split}
\Nm = c_1 \frac{4^\epoch}{(1-\contractPar)^2} \cdot
\log_4(16\epochs\Dims/\pardelta)
\end{split}
\end{align}
\texttt{Sample batches:}
\begin{align}\label{eqn:samplebatches} 
\begin{split}
 &\text{Partition the $\Qnumobs$ samples to obtain $\{\NoisyOptOp_i\}_{i \in \epochsampleset}$ for $m = 1, \ldots, \epochs$}
\end{split}
\end{align}
\texttt{Epoch length:}
\begin{align}\label{eqn:Epoch-length}
\epochlength = \frac{\Qnumobs}{2\epochs}.
\end{align}
\end{subequations}
\begin{props}
\label{prop:vrqlearn-bound}
Suppose the inputs parameters 
of Algorithm~\ref{AlgVR} are chosen according to parameter choices~\eqref{EqnQParam}, and the sample size satisfies the
lower bound $\frac{\numobs}{\NumEpoch} \geq c_1 \frac{\log(8\Dim\NumEpoch/\delta)}{(1 - \contractPar)^3}$. Then, for any initialization $\widebar{\Qvalues}_1$, the output $\widehat{\Qvalues}_{\numobs} \equiv \widebar{\Qvalues}_{\NumEpoch + 1}$
satisfies the 
\begin{align*}
\inftynorm{\qvalueshat_\numobs - \qvaluesstar} &\leq c_1 \cdot \frac{\diagnorm{\optCovMatQ{\qvaluesstar}}^{\frac{1}{2}}}{1 - \contractPar} \cdot \sqrt{\frac{\log(8 \Dims \epochs / \pardelta)}{\Qnumobs}} + c_2 \cdot \frac{\bfun(\qvaluesstar)}{1 - \contractPar} \cdot \frac{\log(8 \Dims \epochs / \pardelta)}{\Qnumobs} \\
&\qquad + c_3 \cdot \inftynorm{\qvalues_1 - \qvaluesstar} \cdot \frac{\log^2((16 \Dims \NumEpoch / \pardelta) \cdot \log \Qnumobs)}{\Qnumobs^2 (1- \contractPar)^4},
\end{align*}
with probability at least $1 - \pardelta$. Here $c_1, c_2, c_3$ are
universal constants.
\end{props}

\subsection{Proof of Proposition~\ref{prop:vrqlearn-bound}}

The proof is this similar to that of the Theorem~2 in the
paper~\cite{khamaru2021instance}; in particular, we use a
modified version of Lemma~7 from the
paper~\cite{khamaru2021instance}.

\subsubsection{Proof set-up}
We start by introducing some notation used in the
paper~\cite{khamaru2021instance}. Recall that the update
within an epoch takes the form (cf.~\ref{AlgRunEpoch})
\begin{align*}
\Qvalues_{k+1} = (1 - \stepsize_k) \Qvalues_k + \stepsize_k \left \lbrace
\NoisyOptOp_k(\Q) - \NoisyOptOp_k(\Qvaluesbar_m) + \RecenterOp(\Qvaluesbar_m) \right
\rbrace,
\end{align*}
where $\Qbar_m$ represents the input into epoch $\epoch$. We define
the shifted operators and noisy shifted operators for epoch $\epoch$
by
\begin{align}
\label{eqn:shifted-bellman}
\shiftedOp(\Qvalues) = \BellOptOp(\Qvalues) - \BellOptOp(\Qbar_m) + \RecenterOp(\Qbar_m) \quad
\text{and} \quad \NoisyShiftedOp_k(\Qvalues) = \NoisyOp_k(\Qvalues) -
\NoisyOptOp_k(\Qbar_m) + \RecenterOp(\Qbar_m).
\end{align}
Since both of the operators $\BellOptOp$ and $\NoisyOptOp_k$ are
$\contractPar$-contractive in the $\ell_\infty$-norm, the operators
$\shiftedOp$ and $\NoisyShiftedOp_k$ are also
$\contractPar$-contractive operators in the same norm.  Let
$\Qhat_\epoch$ denote the unique fixed point of the operator $\shiftedOp$.

With this set-up, it suffices to prove the following modification of
Lemma 7 from the paper~\cite{khamaru2021instance}.
\begin{lemma}
\label{lem:Instance-DependentQ-weak-BD}
Assume that $\Nm$ satisfies the bound \mbox{$ \Nm \geq \frac{c
    \log(8\Dims \epochs / \pardelta)}{(1 - \contractPar)^2}$}. Then we
have
\begin{align*}
\inftynorm{\qvalueshat_\epoch - \qvaluesstar} \leq
\frac{\inftynorm{\qvaluesbar_\epoch - \qvaluesstar}}{33} + c_4 \left\{
\frac{\diagnorm{\optCovMatQ{\qvaluesstar}}^{\frac{1}{2}}}{1 -
  \contractPar} \cdot \sqrt{\frac{\log(8 \Dims \epochs /
    \pardelta)}{\Qrecentersize}} + \frac{\bfun(\qvaluesstar)}{1 -
  \contractPar} \cdot \frac{\log(8 \Dims \epochs /
  \pardelta)}{\Qrecentersize} \right\},
\end{align*}
with probability at least $1 - \frac{\pardelta}{2\epochs}$.
\end{lemma}
Indeed, the proof of Proposition~\ref{prop:vrqlearn-bound} follows 
directly from the proof of Theorem 2 in the paper~\cite{khamaru2021instance} by replacing Lemma 7 by
Lemma~\ref{lem:Instance-DependentQ-weak-BD}. 
\subsubsection{Proof of Lemma~\ref{lem:Instance-DependentQ-weak-BD}}
In order to simplify notation, we drop the epoch number $\epoch$ from
$\qvalueshat_\epoch$ and $\qvaluesbar_\epoch$ throughout the remainder
of the proof.  Let $\pihat$ and $\pistar$ denote the greedy policies
with respect to the $Q$ functions $\qvalueshat$ and $\qvaluesstar$,
respectively. Concretely,
\begin{align}
  \pistar(\state) = \arg\max_{\action \in \actionset} \;\;
  \qvaluesstar(\state, \action) \qquad \pihat(\state) = \arg\max_{\action
    \in \actionset} \;\; \qvalueshat(\state, \action).
\end{align}
Ties in the $\arg\max$ are broken by choosing the
action~$\action$ with smallest index.

By the optimality of the policies $\pihat$ and $\pistar$ for the
$Q$-functions $\qvalueshat$ and $\qvaluesstar$, respectively, we have
\begin{gather}
\label{eqn:Qhat-Qstar-fixed-pt-eqn}
  \Qstar = \rewardQ + \contractPar \TranMatQ^{\pistar} \Qstar \quad
  \text{and} \quad \Qhat = \rewardtilQ + \contractPar
  \TranMatQ^{\pihat} \Qhat, \quad \text{where} \quad
    \rewardtilQ \defn \rewardQ + \RecenterOp(\Qbar) - \BellOptOp(\Qbar).
\end{gather}
Our proof is based on the following intermediate 
inequality which we prove at the end of this section. 
\begin{align}
\label{eqn:fixedpoint-gap-bound}
\inftynorm{\qvalueshat - \qvaluesstar} \leq \frac{1}{1 - \contractPar} \inftynorm{\rewardtilQ - \rewardQ}.
\end{align}
With the last inequality at hand it suffices to prove an upper bound on the term $\inftynorm{\rewardQ - \rewardQ}$.

Recall the definition \mbox{$\rewardtilQ \defn \rewardhatQ +
  \contractPar (\obsmathatQ^{\pibar} - \TranMatQ^{\pibar}) \Qbar$,}
where $\pibar$ a policy greedy with respect to $\Qbar$; that is,
$\pibar(\state) = \arg \max_{\action \in \actionset} \Qbar(\state,
\action)$, where we break ties in the $\arg\max$ by choosing the
action~$\action$ with smallest index. We have
\begin{align*}
\inftynorm{ \rewardtilQ - \rewardQ} &\leq \inftynorm{ (\rewardhatQ -
  \rewardQ) + \contractPar(\obsmathatQ^{\pistar} -
  \TranMatQ^{\pistar}) \Qstar } + \contractPar \inftynorm{
  (\obsmathatQ^{\pibar} \Qbar - \obsmathatQ^{\pistar} \Qstar) -
  (\TranMatQ^{\pibar} \Qbar - \TranMatQ^{\pistar} \Qstar)}.
\end{align*}

Observe that the random variable $\rewardhatQ$ and $\obsmathatQ$ are
averages of $\Nm$ i.i.d.~random variables $\{\NoisyReward_i\}$ and
$\{\obsmatQ_i \}$, respectively. Consequently,
applying Bernstein bound 
along with a union bound we have the following 
bound with probability 
least $1 - \frac{\pardelta}{4
  \epochs}$:
\begin{multline*}
\inftynorm{ (\rewardhatQ - \rewardQ) + \contractPar(\obsmathat^{\pistar}
  - \TranMatQ^{\pistar}) \Qstar } \leq \frac{4}{\sqrt{\Nm}} \cdot
\diagnorm{\optCovMatQ{\qvaluesstar}}^{\frac{1}{2}} \cdot \sqrt{\log(8\Dim\NumEpoch/\delta)}
\\
+ \frac{4 \bfun(\qvaluesstar)}{(1 - \contractPar) \Nm} \cdot
\log(8\Dim\NumEpoch/\delta).
\end{multline*}
Finally, for each state-action pair $(\state, \action)$ the random
variable $(\obsmathat^{\pibar} \Qbar - \obsmathat^{\pistar}
\Qstar)(\state, \action)$ has expectation $(\TranMatQ^{\pibar} \Qbar -
\TranMatQ^{\pistar})(\state, \action)$ with entries uniformly bounded
by $2 \inftynorm{\Qbar - \Qstar}$. Consequently, by a standard
application of Hoeffding's inequality combined with the lower bound
$\Nm \geq c_3 \frac{4^\epoch}{(1 - \contractPar)^2} \log(8\Dim\NumEpoch/\delta)$, we have
\begin{align*}
\frac{\contractPar}{1 - \contractPar} \cdot
\inftynorm{(\obsmathat^{\pibar} \Qbar - \obsmathat^{\pistar} \Qstar) -
  (\TranMatQ^{\pibar} \Qbar - \TranMatQ^{\pistar} \Qstar)} \leq
\frac{\inftynorm{\Qbar - \Qstar}}{33},
\end{align*}
with probability at least $1 - \frac{\pardelta}{4\epochs}$. The
statement of Lemma~\ref{lem:Instance-DependentQ-weak-BD} then follows
from combining these two high-probability statements with a union
bound.  It remains to prove the
claim~\eqref{eqn:fixedpoint-gap-bound}.

\paragraph{Proof of equation~\eqref{eqn:fixedpoint-gap-bound}:}

By optimality of the policies $\pihat$ and $\pistar$ for the
$Q$-functions $\Qhat$ and $\Qstar$, respectively, we have
\begin{align}
  \Qstar = \rewardQ + \contractPar \TranMatQ^{\pistar} \Qstar \geqEle
  \rewardQ + \contractPar \TranMatQ^{\pistar} \Qhat \quad \text{and}
  \quad \Qhat = \rewardtilQ + \contractPar \TranMatQ^{\pihat} \Qhat
  \geqEle \rewardtilQ + \contractPar \TranMatQ^{\pihat} \Qstar.
\end{align}
Thus, we have 
\begin{align}
\label{eqn:q-gap-one}
  \Qstar - \Qhat &= \rewardQ - \rewardtilQ +
  \contractPar(\TranMatQ^{\pistar} \Qstar - \TranMatQ^{\pihat} \Qhat)
  \leqEle \rewardQ - \rewardtilQ + \contractPar\TranMatQ^{\pistar}
  (\Qstar - \Qhat).
\end{align}
Rearranging the last inequality, and using the non-negativity of the
entries of $(\Id - \contractPar\TranMatQ^{\pistar})^{-1}$ we conclude
\begin{align*}
  (\Qstar - \Qhat) \leqEle (\Id -
  \contractPar\TranMatQ^{\pistar})^{-1}(\rewardQ - \rewardtilQ).
\end{align*}
This completes the proof of the bound; a similar argument gives
\begin{align*}
(\Qhat - \Qstar) \leqEle (\Id - \contractPar \TranMatQ^{\pihat})^{-1}(\rewardQ - \rewardtilQ).
\end{align*}
Collecting the two bounds, we have
\begin{align*}
| \Qhat - \Qstar| \leqEle \max\left\{(\Id -
  \contractPar\TranMatQ^{\pistar})^{-1}(\rewardQ - \rewardtilQ), (\Id - \contractPar \TranMatQ^{\pihat})^{-1}(\rewardQ - \rewardtilQ) \right\}, 
\end{align*}
where $\max$ denotes the entry-wise maximum. The desired then follows from the fact
the bound $\| (\Id - \contractPar \TranMatQ^{\pi})^{-1} \|_{1, \infty} \leq \frac{1}{1 - \contractPar}$ for any policy $\pi$. This completes the proof. 
\end{document}